\documentclass[a4paper]{article}

\usepackage{natbib}
\usepackage{graphicx}
\usepackage{amsmath,amsthm,amssymb}
\theoremstyle{plain}
\newtheorem{thm}{Theorem}     
\newtheorem{lem}{Lemma}         
\newtheorem{prop}{Proposition} 

\usepackage{enumitem}
\usepackage{algorithm}
\usepackage{algorithmic}

\usepackage{mathrsfs}
\usepackage{mathtools}
\usepackage[a4paper,margin=1in]{geometry}
\usepackage{xcolor}
\usepackage{hyperref} 
\hypersetup{
    colorlinks=true,       
    linkcolor=blue,        
    citecolor=green!50!black, 
    urlcolor=cyan,         
}

\newtheoremstyle{namedassumption}
  {\topsep}{\topsep}
  {\itshape}{}
  {\bfseries}{.}
  {.5em}
  {\thmnote{(\textbf{#3})}} 

\newtheoremstyle{namedassumption}
  {\topsep}{\topsep}
  {\itshape}{}
  {\bfseries}{}
  {.5em}
  {(\thmnote{#3})} 

\theoremstyle{namedassumption}
\newtheorem*{namedass}{}

\def\R{\mathbb{R}}
\def\Rd{\mathbb{R}^d}
\def\Z{\mathcal{Z}}
\def\FZ{\mathcal{F}_{\Z}}
\def\P{\mathcal{P}}
\def\d{\mathrm{d}}
\def\Sd{\mathbb{S}^{d-1}}
\def\data{\pi_{\mathrm{data}}}

\def\BR{\mathcal{B}(\Z,\Rd)}
\def\BS{\mathcal{B}(\Z,\Sd)}

\def\bz{\mathbf{Z}}

\newcommand{\simiid}{\overset{\mathrm{iid}}{\sim}}

\DeclareMathOperator*{\argmin}{\arg\!\min}

\title{Similarity search generalisation in contrastive learning with InfoNCE loss}
\author{Nick Whiteley\\
School of Mathematics, University of Bristol, U.K.}

\begin{document}

\maketitle
\begin{abstract}
Similarity search is a primary application of embedding models trained by contrastive learning. For one of the most popular contrastive learning loss functions, InfoNCE, we show that the  population risk with $k$ negative samples is $O(1/k)$ close to an expected cross-entropy which quantifies deviation between i) a softmax similarity search over unseen data using the learned embedding function, and ii) an idealised softmax search over the same data but using similarity implicitly represented in the positive sample generator. This complements existing interpretations of InfoNCE in the $k\to\infty$ limit which are phrased in terms of mutual information,  and alignment versus uniformity in embeddings. To quantify generalisation performance, we introduce a new continuity bound for the InfoNCE loss, obtained via G\^{a}teaux differentiation. The bound preserves the structure of averaging over negative samples present in the loss function and features an ``inverse temperature'' parameter which can be tuned to account for the algorithmic temperature. For embedding functions which are Lipschitz in a parameter, this yields a simple demonstration that the averaging effect of $k$ negative samples in the InfoNCE loss carries over to stabilisation of the generalisation error as $k$ grows.

\end{abstract}

\section{Introduction}

The InfoNCE loss \citep{vandenOord2018cpc} is one of the most popular loss functions in contrastive learning and is a fundamental ingredient in hugely impactful systems such as SimCLR \citep{chen2020simclr}, MoCo \citep{he2020moco} and CLIP \citep{radford2021clip}. Similarity search is a primary application of these technologies and other embedding models trained using contrastive learning; the learned embedding is used to calculate cosine similarities and hence evaluate closeness among unseen data. The goal of the present work is to add clarity to our theoretical understanding of InfoNCE, similarity search and generalisation.

To date, theoretical generalisation analysis of contrastive learning has largely focused on downstream classification. In that context, ``generalisation'' has the conventional meaning of a model's ability to make accurate predictions on unseen data:  a pioneering step forward was made by \citet{saunshi2019theoretical}, who showed that the InfoNCE population risk can contribute to  bounding the population risk of a downstream linear classifier, thus quantifying classification accuracy when an unseen input data point, such as an image or text document, is represented by its embedding vector. Subsequent refinements and extensions, discussed in more detail later, have been made by \citep{lei2023generalization,ghanooni2024generalization,  hieu2025generalization}.

In the present work we also analyse InfoNCE risk but from a different perspective, which we call \emph{similarity search generalisation}: the performance of similarity search on unseen data using the learned embedding, compared to an idealised search which we show is implicitly defined by the ingredients of contrastive learning.  We draw inspiration not only from learning-theoretic analyses of e.g., \citep{saunshi2019theoretical,lei2023generalization,  ghanooni2024generalization, hieu2025generalization}, but also from widely referenced interpretations of contrastive learning in terms of mutual information \citep{vandenOord2018cpc}, ``alignment'' versus ``uniformity'' in embeddings \citep{wang2020understanding}, and cross-entropy between conditional probability densities \citep{zimmermann2021contrastive}.

\subsection{Interpretations  of InfoNCE}
Each training tuple or mini-batch in contrastive learning comprises  ``anchor'', ``positive'' and ``negative'' samples\footnote{The terms ``anchor'', ``positive'' and ``negative'' samples are standard in the contrastive learning literature. With a warning to the uninitiated reader that ``positive'' and ``negative'' should not be interpreted in the conventional mathematical sense, the quotation marks will be dropped from here on.}. Pairs of anchor and positive samples are generated in order to convey some notion of semantic similarity which is specific to input data modality, such as text, images, etc. Negative samples are usually additional data points which are independent of the anchor and positive samples.  

Conventional theoretical interpretations of contrastive learning with InfoNCE loss address the regime where the number of negative samples grows. The InfoNCE loss function was introduced by \citet{vandenOord2018cpc} as the categorical cross-entropy associated with correctly classifying a positive sample versus  negative samples. They showed that minimising  InfoNCE risk maximises a lower bound on mutual information, and argued that bound becomes tight as the number of negative samples grows. \citet{wang2020understanding} uncovered another interpretation of InfoNCE associated with taking the number of i.i.d. negative samples $k\to\infty$. They showed that, in this limit, the logarithmically normalised InfoNCE population risk converges to a limiting risk function which separates into a sum of two terms, arguing that minimising InfoNCE loss promotes a balance between ``alignment'' and ``uniformity'' of embeddings.  The same 
$k\to\infty$ limit was given  another interpretation by \citet{zimmermann2021contrastive}, as cross-entropy between conditional probability density  functions, in a setting where the input samples are assumed to be drawn from a distribution defined by pushing the uniform density on a hypersphere or convex subset of Euclidean space through an invertible transformation. They argued that contrastive learning can invert that data generating model.

In all of these interpretations, the convergence of the logarithmically normalised InfoNCE population risk as $k\to\infty$ can be viewed as a consequence of a law of large numbers for the negative samples. The risk converges because averaging over these negative samples occurs within the loss function. In the  present work we explore yet another interpretation of the $k\to\infty$ averaging, relating InfoNCE population risk to similarity search on unseen data.

\subsection{Generalisation analysis of contrastive learning}

In order to quantify downstream classification performance, \cite{saunshi2019theoretical} and in turn \citet{lei2023generalization} assumed a specific data-generating model involving discrete classes and class labels, with positive samples generated by drawing from the same class-conditional distribution as the anchor sample. However, their complexity analysis of InfoNCE risk (rather than classification risk) depends very little on this model. To apply a standard generalisation bound \citep[Thm. 3.3]{mohri2018foundations} they only really require that training tuples are i.i.d., and their empirical complexity bounds hold for any realisation of the data hence do not require any distributional assumptions at all. We would like to exploit this fact in order to apply their results outside of the context of classification.

Considering the interest in the regime $k\to\infty$ described above, this naturally leads us to the question of how InfoNCE generalisation error  behaves as $k$ grows. As discussed by \citet{lei2023generalization}, the complexity bounds of \citet{saunshi2019theoretical} feature an explicit factor which grows with $k$ (e.g.,  see the factor of $\sqrt{k}$ in \citep{saunshi2019theoretical}[Supp. material, eq. (30)]),  seemingly at odds with studies indicating that a large number of negative samples is necessary for good generalisation performance in practice \citep{chen2020simclr, he2020moco, henaff2020data, khosla2020supervised, tian2020contrastive}.

\citet{lei2023generalization} revisited this generalisation analysis and used alternative mathematical techniques to obtain refined complexity bounds, improving over \citep{saunshi2019theoretical} by eliminating the explicit factor of $\sqrt{k}$ for $\ell_2$-Lipschitz loss functions and reducing by a further factor of $\sqrt{k}$ for $\ell_\infty$-Lipschitz losses (up to factors logarithmic in $k$ and the number of training tuples $n$). \citet{lei2023generalization}  also obtained data-dependent bounds in the case of self-bounding loss functions.  Nevertheless, the complexity estimates obtained by \citet{lei2023generalization} for $\ell_2$ or $\ell_\infty$-Lipschitz loss functions still depend on $k$ and may grow with $k$ in general; they involve summing and/or maximising over an index set which grows with $k$ (see the quantities denoted $\mathfrak{A}$ and $\mathfrak{C}$ in \citep{lei2023generalization}[eqs. 4.2 and 4.5]). The resulting bounds for linear and nonlinear features \citep{lei2023generalization}[sec. 5, quantity $B_x$] involve the maximum of the norms of the input data vectors, which will grow with $k$ in general when the input domain is unbounded. One of the aims of the present work is to clarify whether such dependence on $k$ is avoidable.

In more recent developments, the assumption of i.i.d. data tuples made by \citet{saunshi2019theoretical} and \citet{lei2023generalization}  was loosened by \citet{hieu2025generalization}, bringing the analysis closer to practical dependency between samples using $U$-statistics. \citet{ghanooni2024generalization} developed generalisation analysis for adversarial contrastive learning.

\subsection{Outline and contributions}\label{sec:contributions}

\begin{itemize}[leftmargin=*]
\item Section \ref{sec:contrastive_learning} presents the basic ingredients of  contrastive learning with the InfoNCE loss. Our setup is purposefully general in some ways: our measure-theoretic notation is chosen to help interpret the InfoNCE population loss in terms of Markov kernels and softmax similarity search. However, our presentation is purposefully narrow in other ways: our priority is to give the reader, in just a few pages, an end-to-end account from interpretation of the InfoNCE loss,  to easily interpretable generalisation bounds which exhibit the role of $k$ and other quantities. We thus introduce contrastive learning as obtaining an embedding function $\phi$ by empirical risk minimisation (although strict minimisation is not required for our generalisation results to apply), and do not enter into details of how specific neural network architectures, gradient algorithms, mini-batches etc., are used in practice. 
\item In section \ref{sec:sec:sim_search_basics} we introduce a Markov kernel $Q_{\tau}^\phi$ on the input space $\Z$, depending on the embedding function $\phi$ and temperature parameter $\tau$, and in proposition \ref{prop:bias_simple} in section \ref{sec:pop_risk_int_cross} show that as $k\to\infty$, the (logarithmically normalised) InfoNCE population risk converges to  cross-entropy between $Q_{\tau}^\phi$ and the Markov kernel $M$ which generates positive samples: $Z^+\sim M(Z,\cdot)$ where $Z\sim\data$ and $\data$ is a probability measure on $\Z$. This presentation is partly inspired by \citep[Thm 1.]{zimmermann2021contrastive} but does not require the specifics of their data-generating model. It is already known that the population risk converges in this limit, \citet[Thm 1.]{wang2020understandingv10} report a rate $O(1/\sqrt{k})$; we clarify in proposition \ref{prop:bias_simple} the rate is $O(1/k)$. We also highlight the regularising role of the temperature parameter: the higher $\tau$ is, the more $Q_\tau^\phi(z,\cdot)$ is constrained to be close to $\data$, uniformly in $z$ and $\phi$.
\item In section \ref{sec:pop_risk_expected} we  present an interpretation of the InfoNCE population risk which, to the knowledge of the author, is new. We introduce an empirical softmax similarity  Markov kernel $\widehat{Q}_{\tau,k}^\phi$  defined in terms of  $k$ unseen (i.e., independent of training data) draws from $\data$. In proposition \ref{prop:Q_hat} we show that when $M(z,\cdot)$ is dominated by $\data$ for all $z\in \Z$, the (logarithmically normalised) InfoNCE population risk is $O(1/k)$ close to the expected cross-entropy between $\widehat{Q}_{\tau,k}^\phi$ and an idealised empirical Markov $\widehat{M}_k$ kernel built from the same $k$ unseen data points importance weighted according to the density of $M(z,\cdot)$ with respect to $\data$, where $z$ is the search query point. In this sense $\widehat{M}_k$ conveys whatever notion of similarity is implicitly represented in $M$, and sampling from $\widehat{M}_k$ has the interpretation of an idealised softmax similarity search.
\item Motivated by these considerations of InfoNCE when $k$ is large, we turn to generalisation analysis in section \ref{sec:generalisation}. The key mathematical contribution in  section \ref{sec:gen_and_cont} is a new  continuity bound for the InfoNCE loss, presented in proposition \ref{prop:ell_lipschitz}. This bound is obtained by via G\^{a}teaux differentiation, exploiting the specific structure of the InfoNCE loss, whereas  \citet{lei2023generalization}'s analysis applies more generally to $\ell_2/\ell_\infty$-Lipschitz or self-bounding loss functions. The continuity bound is applied to bounding InfoNCE Rademacher complexity in section \ref{sec:app_to_lipschitz_functions} when the embedding function is chosen from a class of functions Lipschitz in a parameter. This demonstrates the structure of averaging over $k$ negative samples in the InfoNCE loss carries over to stabilisation of generalisation error as $k$ grows.
\item Possible extensions are discussed in section \ref{sec:discussion}. All proofs are in the appendix.
\end{itemize}

\subsection{Notation}

Throughout this work,  $\mathcal{Z}$ is a set and $\mathcal{F}_{\mathcal{Z}}$ is a $\sigma$-algebra
of subsets of $\mathcal{Z}$. The set of probability measures on the measurable space $(\Z,\FZ)$ is denoted $\P(\Z)$. 
The delta-Dirac measure located at $z\in\Z$ is denoted $\delta_{z}$.  For $d\geq2$, $p\geq1$ and $v=(v_{1},\ldots,v_{d})\in\mathbb{R}^{d}$
we write the norm $\|v\|_{p}\coloneqq(\sum_{j=1}^{d}|v_{j}|^{p})^{1/p}$,
and denote by $\langle\cdot,\cdot\rangle$ the Euclidean inner product. We write $\Sd$ for the set all of $v\in\Rd$
such that $\|v\|_{2}=1$. For functions $f:\mathcal{Z}\to\mathbb{R}^{d}$ we define the norm $\|f\|_{2,\infty}\coloneqq\sup_{z\in\Z}\|f(z)\|_{2}$ and denote by $\BR$ the Banach space of all measurable $f:\mathcal{Z}\to\mathbb{R}^{d}$ such that $\|f\|_{2,\infty}<\infty$. We denote by $\BS \subset \BR$ the set of 
those functions $f$ 
such that $\|f(z)\|_{2}=1$ for all $z\in Z$. For scalars $a,b$, the maximum and minimum are denoted $a\vee b$ and $a\wedge b$, respectively.

\section{Contrastive learning with the InfoNCE loss}\label{sec:contrastive_learning}

\subsection{Preliminaries}

The model we assume for contrastive learning training data has three ingredients:
\begin{itemize}
\item a set $\Z$;
\item a probability measure $\data\in\P(\Z)$;
\item a Markov kernel   $M:\Z\times\FZ\to[0,1]$, i.e., for each $z\in\Z$, $M(z,\cdot)\in\P(\Z)$,  and for each $A\in\FZ$, $z\mapsto M(z,A)$ is a measurable function.
\end{itemize}
The training data comprises $n\geq 1$ tuples, $\bz_1,\ldots,\bz_n$, where, for $k\geq 1$, each tuple $\bz_i \in \Z^{2+k}$ consists of an anchor sample, $Z_i^a$; a positive sample, $Z_i^+$; and negative samples, $Z_{i1}^-,\ldots,Z_{ik}^-$.  The following is taken as a standing assumption throughout the entirety of this work.
\begin{namedass}[A]\label{ass:dist2}
$\bz_1,\ldots,\bz_n$ are i.i.d. and each tuple  $\bz_i$ is distributed: $Z_i^a\sim \data$, $Z_i^+\sim M(Z_i,\cdot)$, $Z_{i1}^-,\ldots,Z_{ik}^-\simiid \data$.
\end{namedass}

In practice, positive samples are generated by applying various transformations to the anchor samples, sometimes called ``views''. Examples include, for images: cropping, resizing, rotation, colour adjustment, jitter and blurring; and for text: synonym replacement, back-translation, and swapping, insertion and deletion of characters, tokens, or words. All of these transformations are usually  subject to some degree of randomisation, for example randomly choosing the area to be cropped, the rotation angle, etc., and the choice of transformation, or the order in which to compose transformations, may also be randomised. We assume that any and all such randomisation is encapsulated in the Markov kernel $M$. We do not assume any particular functional form or algorithmic description of $M$ is known; for our purposes it can be thought of as a ``black box'' sample generator (later subject to assumption $\mathrm{(\nameref{ass:dom})}$, which is applied solely within section \ref{sec:sim_search}). Similarly, we do assume any particular functional form for $\data$, in practice this distribution is unknown.

For our purposes, we may think of contrastive learning as choosing an embedding function $\widehat{\phi}$  from some class $\Phi\subset \BS$ by minimising the empirical InfoNCE risk associated with $\mathbf{Z}_1,\ldots,\mathbf{Z}_n$. The procedure is outlined in algorithm \ref{alg:InfoNCE}. Later on, in section \ref{sec:generalisation}, we will consider specific choices of $\Phi$.  
\begin{algorithm}
\caption{InfoNCE Empirical Risk Minimization \label{alg:InfoNCE}}
\begin{algorithmic}
\STATE {\bf inputs}: anchor samples  $Z_1^a,\ldots,Z_n^a\simiid \data$; integer $k\geq1$, set of functions $\Phi$; temperature param. $\tau>0$.	
\FOR{$i=1,\ldots,n$, }
\STATE -- draw one positive sample $Z_i^+\sim M(Z_{i}^a,\cdot)$
\STATE -- draw $k$ negative samples $Z_{i1}^-,\ldots,Z_{ik}^- \simiid \data$
\ENDFOR
\STATE {\bf do} minimisation:
$$
\widehat{\phi} = \argmin_{\phi\in\Phi} \sum_{i=1}^n  \log\left(1+ \sum_{j=1}^k \exp\left[\langle \phi(Z_i^a),\phi(Z_{ij}^-)-\phi(Z_i^+)\rangle/\tau \right] \right).
$$
\end{algorithmic}
\end{algorithm}

The empirical and population InfoNCE risk functionals, $\widehat{R}_n(\cdot;k,\tau)$ and $R(\cdot;k,\tau)$, are defined as follows.  For $\phi\in \BS$, $k\geq 1$ and $\tau>0$, 
\begin{align}
\widehat{R}_n(\phi;k,\tau) &\coloneqq \frac{1}{n}\sum_{i=1}^n  \log \left(1+ \sum_{j=1}^k e^{\langle\phi(Z_i^a,\phi(Z_{ij}^-) - \phi(Z_i^+)\rangle/\tau} \right) - \log k,\label{eq:R_hat_defn}  \\
R(\phi;k,\tau) &\coloneqq \mathbb{E}\left[ \log \left(1 + \sum_{j=1}^k e^{\langle\phi(Z^a),\phi(Z_{j}^-) - \phi(Z^+)\rangle/\tau}\right)\right] - \log k,\label{eq:R_k_defn} 
\end{align}
where in \eqref{eq:R_k_defn}, expectation is over $Z^a\sim \data,  Z^+ \sim M(Z^a,\cdot)$, $Z^-_{1},\ldots,Z^-_{k}\simiid \data$, so  $\mathrm{(\nameref{ass:dist2})}$ implies $R(\phi;k,\tau) = \mathbb{E} [ \widehat{R}_n(\phi;k,\tau) ]$. The empirical risk $\widehat{R}_n(\phi;k,\tau)$  differs from the quantity in Algorithm \ref{alg:InfoNCE} only by the multiplicative $1/n$ and logarithmic $-\log k$ normalisation factors, which the $\argmin$ operation is invariant to.

\subsection{Discussion of the setup}\label{sec:discuss_the_setup}

\begin{itemize}[leftmargin=*]
\item The assumption within $\mathrm{(\nameref{ass:dist2})}$ that the tuples $\mathbf{Z}_1,\ldots,\mathbf{Z}_n$ are i.i.d. was made in \citep{saunshi2019theoretical,lei2023generalization}. Similarly to those works, we use the assumption of  i.i.d. tuples to apply a standard generalisation bound \citep[Thm. 3.3]{mohri2018foundations} in the proof of  proposition \ref{prop:uniform},  section \ref{sec:generalisation}. The assumption that, within each tuple, negative samples are i.i.d. from the same distribution as the anchor samples is quite common in the literature, e.g., \citep{wang2020understanding,zimmermann2021contrastive}. This assumption of i.i.d. negative samples is used in the proofs of propositions \ref{prop:bias_simple} and \ref{prop:Q_hat} in section \ref{sec:sim_search}, concerning interpretation of the InfoNCE population risk. However, this part of $\mathrm{(\nameref{ass:dist2})}$ is not needed for any of the results in section \ref{sec:generalisation} concerning generalisation. In fact, propositions \ref{prop:ell_lipschitz} and lemma \ref{lem:metric}  in section \ref{sec:generalisation}, used to bound empirical Rademacher complexity, hold for any realised values $\mathbf{z}_1,\ldots,\mathbf{z}_n$ of $\mathbf{Z}_1,\ldots,\mathbf{Z}_n$, and hence do not require assumption $\mathrm{(\nameref{ass:dist2})}$ at all.
\item In some presentations of contrastive learning and the InfoNCE loss, it is assumed that the joint distribution of anchor and positive pairs has a  symmetric density and/or that the marginal distributions of the positive and negative samples are the same, e.g., \citep{vandenOord2018cpc,wang2020understanding}. We make no such assumptions. Indeed the author considers it unrealistic to assume that if $Z\sim\data$ and $Z^+\sim M(Z,\cdot)$ then the marginal distribution of $Z^+$ is  exactly $\data$. This would amount to saying that $M$ admits $\data$ as an invariant distribution; in practice the distribution $\data$ is unknown and nothing is built in to achieve such invariance.   
\item A temperature parameter was not presented in the InfoNCE loss by \citep{vandenOord2018cpc} but was included in, e.g., \cite{wu2018unsupervised,chen2020simclr}, where the loss is sometimes called the normalized
temperature-scaled cross entropy loss (NT-Xent). In their generalisation analyses, \citet{saunshi2019theoretical},\citet{lei2023generalization} did not explicitly consider a temperature parameter, but considered embedding functions $\phi $ satisfying $\|\phi\|_{2,\infty}\leq R$ for some finite $R>0$. In the present work our embedding functions are always members of $\BS$, i.e., $\|\phi(z)\|_2=1$ for all $z\in \Z$. This is not essential for many of our results, but we make this assumption to conform with common practice and presentation of contrastive learning; for $\phi\in\BS$, $\langle\phi(z),\phi(z^\prime)\rangle$ is the cosine similarity between $\phi(z)$ and $\phi(z^\prime)$.  

\item  For the purposes of our analysis it will not be important that $\widehat{\phi}$ is an exact minimiser as in algorithm \ref{alg:InfoNCE}, we write it as such just for sake of illustration; our various bounds presented in sections \ref{sec:sim_search} and \ref{sec:generalisation} hold uniformly over embedding functions belonging $\BS$ or some some $\Phi\subset\BS$. 
\end{itemize}

\section{Similarity search interpretation of the InfoNCE loss}\label{sec:sim_search}
\subsection{Basics of similarity search}\label{sec:sec:sim_search_basics}
Suppose that we are given an embedding function $\phi\in \BS$, for example  $\widehat{\phi}$ from algorithm \ref{alg:InfoNCE}, and we use it for similarity search on unseen data, as follows: for some $m\geq1$, let $\xi_1,\ldots,\xi_m$ be points in $\Z$ (we use the notation $\xi_i$ to distinguish these samples from any of the constituents of the training tuples $\mathbf{Z}_1,\ldots,\mathbf{Z}_n$) and let $z\in\Z$ be a query point. Similarity search is the task:
\begin{equation}\label{eq:sim_search}\text{find  }\quad\xi_\star\in\{\xi_1,\ldots,\xi_m\}\quad \text{ which maximises }\quad \langle\phi(z),\phi(\xi_\star)\rangle.
\end{equation}

A ``softmax'', randomised relaxation of this similarity search with the same query point $z$ is to  instead sample $\xi_\star$ from the set $\{\xi_1,\ldots,\xi_m\}$ as follows:
\begin{equation}\label{eq:soft_sim_search}
\text{set }\quad \xi_\star=\xi_i\quad \text{ with probability  }\quad\frac{W_i^{1/\tau}(z)}{\sum_{j=1}^m W_j^{1/\tau}(z)},\quad \text{where}\quad W_i(z)\coloneqq e^{\langle \phi(z),\phi(\xi_i)\rangle}.
\end{equation}
When $\tau\to 0$, \eqref{eq:soft_sim_search} reduces to \eqref{eq:sim_search}. Our next objective is to explain the connection between \eqref{eq:soft_sim_search} and the InfoNCE population risk $R(\phi;k,\tau)$. We need some further definitions. For any $\phi\in \BS$ and $\tau>0$, define the Markov kernel\footnote{A closely related but technically different conditional probability density was introduced by \citet{zimmermann2021contrastive}[Thm. 1] in the setting of a latent variable model for inputs to contrastive learning. We note their purposes were somewhat different to ours, focusing on the question of whether contrastive learning can invert a data generating process, rather than focusing on similarity search.}:
\begin{equation}\label{eq:Q_defn}
Q_\tau^\phi(z,A)\coloneqq \frac{\displaystyle\int_{A}e^{\langle \phi(z),\phi(z^\prime)\rangle / \tau}\data(\d z^\prime)}{\displaystyle\int_{\Z}e^{\langle \phi(z),\phi(z^\prime)\rangle / \tau}\data(\d z^\prime)},\qquad z\in\Z,\,A\in\FZ.
\end{equation}

We also introduce the following  empirical counterpart of $Q_{\tau}^\phi$ built from the unseen data $\xi_1,\ldots,\xi_m$,
\begin{equation}\label{eq:Q_hat_defn}
\widehat{Q}_{\tau,m}^\phi(z,A)\coloneqq \frac{\sum_{i=1}^m e^{\langle\phi(z),\phi(\xi_i)\rangle/\tau}\delta_{\xi_i}(A)}{\sum_{i=1}^m e^{\langle\phi(z),\phi(\xi_i)\rangle/\tau}},\quad z\in\Z,\,A\in\FZ.
\end{equation}
These Markov kernels have the following interpretations: sampling $\xi_\star \sim \widehat{Q}_{\tau,m}^\phi(z,\cdot)$ is equivalent to \eqref{eq:soft_sim_search}. In that sense, $\widehat{Q}_{\tau,m}^\phi$ just encapsulates the probabilities in the softmax similarity search. If the unseen data are drawn from $\data$, i.e., $\xi_1,\ldots,\xi_m\simiid \data$, then by the strong law of large numbers, for any $z\in\Z$ and $A\in\FZ$, $ \widehat{Q}_{\tau,m}^\phi(z,A)$ converges almost surely to $ Q_{\tau}^\phi(z,A)$ as $m\to\infty$. Thus $Q_{\tau}^\phi$ captures the behaviour of the softmax similarity search \eqref{eq:soft_sim_search}  in the limit of a large amount of unseen data.

\subsection{Relating population risk to integrated cross-entropy}\label{sec:pop_risk_int_cross}

We define the cross-entropy between $M(z,\cdot)$ and $Q_{\tau}^\phi(z,\cdot)$ (with $\data$ taken as a dominating measure for the latter) as: 
\begin{align}
\mathrm{CrossEnt}[M(z,\cdot)\| Q_\tau^\phi(z,\cdot)] & \coloneqq - \int_{\Z} \log \left[\frac{\d Q_\tau^\phi(z,\cdot)}{\d \data}(z^\prime)\right]\, M(z,\d z^\prime)\label{eq:cross_ent_defn}\\
&= -\frac{1}{\tau}\int_{\Z} \langle \phi(z),\phi(z^\prime) \rangle M(z,\d z^\prime) + \log \int_{\Z} e^{\langle \phi(z),\phi(z^\prime) \rangle/\tau} \data(\d z^\prime),\nonumber
\end{align}
where $\frac{\d\cdot}{\d\cdot}$ denotes Radon-Nikodym derivative \footnote{Strictly speaking, $e^{\langle\phi(z),\phi(z^\prime)\rangle/\tau}/\int_{\Z} e^{\langle\phi(z),\phi(x\rangle/\tau}\data(\d x)$ is a version of the R.-N. derivative $\frac{\d Q_\tau^\phi(z,\cdot)}{\d \data}(z^\prime)$, and we are choosing to make the stated definition of cross-entropy in terms of this specific version.}, and the second equality follows from the definition \eqref{eq:Q_defn}. The following proposition relates this cross-entropy, integrated with respect to $\data$, to the population risk $R(\phi;k,\tau)$.
\begin{prop}\label{prop:bias_simple}For any $\phi\in \BS$, $k\geq 1$ and $\tau>0$,
$$
-\log\left(1+\frac{e^{2/\tau}}{k}\right)\leq \int_{\Z}\mathrm{CrossEnt}[M(z,\cdot)\| Q_\tau^\phi(z,\cdot)]\data(\d z)- R(\phi;k,\tau)  \leq  \frac{1}{8k}(e^{2/\tau }-1)^2.
$$
\end{prop}
Proposition \ref{prop:bias_simple} implies that, as $k\to\infty$,   $R(\phi;k,\tau)$ converges to the $\data$-integrated cross-entropy between $M(z,\cdot)$ and $Q_{\tau}^\phi(z,\cdot)$.  Modulo the technical details of our measure-theoretic setup, it is already well known that $R(\phi;k,\tau)$ approaches a limit as $k\to\infty$, e.g., \citep{vandenOord2018cpc,wang2020understanding, zimmermann2021contrastive}. However, noting that for $x\geq 0$, $\log(1+x)\leq x$, proposition \ref{prop:bias_simple} implies a rate $O(1/k)$ uniformly over $\phi\in\BS$, whereas a rate $O(1/\sqrt{k})$ was reported by \citep{wang2020understandingv10}.

What more does proposition \ref{prop:bias_simple} tell us? The integrated cross-entropy considered in proposition \ref{prop:bias_simple} quantifies  discrepancy between the Markov kernels $M$ and $Q^{\phi}_\tau$.
Thus proposition \ref{prop:bias_simple} indicates that, if we were to hypothetically choose $\phi$ by minimising $R(\phi;k,\tau) $ when $k$ is large, this would amount to minimising discrepancy between $M$ and $Q^{\phi}_\tau$. We might therefore take the view that choosing $\phi$ by minimising InfoNCE risk amounts, in effect, to learning the Markov kernel $M$. Consideration of $Q^{\phi}_\tau$ sheds some light on the role of the temperature parameter $\tau$ here.  Since $\|\phi(z)\|_2=1$ for all $z$, the following two inequalities follow from the definition of $Q_\tau^\phi$ in \eqref{eq:Q_defn} and hold for all $z\in\Z$ and $A\in \FZ$,
\begin{equation}\label{eq:pi_Q_pi}
e^{-2/\tau} \data(A)\leq Q_\tau^\phi(z,A) \leq e^{2/\tau } \data(A). 
\end{equation}
Thus we  see that when $\tau\to\infty$, $Q_{\tau}^\phi(z,\cdot)$ is constrained to be closer and closer to $\data$, for all $z\in\Z$, no matter what the choice of $\phi\in\BS$.  In this sense, a large value of $\tau$ limits the ability of $ Q_{\tau}^\phi$ to closely approximate an arbitrary $M$. 

\subsection{Relating population risk to expected empirical cross-entropy}\label{sec:pop_risk_expected}

Even if we are given $\phi$, the Markov kernel $Q_\tau^{\phi}$ is not available in practice because the data-generating distribution $\data$ is unknown. Let us now look more closely at the empirical Markov kernel $\widehat{Q}_{\tau,k}^\phi$,  defined in \eqref{eq:Q_hat_defn} which we \emph{can} access in practice for any given $\phi$ and $\xi_1,\ldots,\xi_m$. Consider the domination assumption: 
\begin{namedass}[B]\label{ass:dom} for all $z\in\Z$, $M(z,\cdot)\ll \data$,
\end{namedass}
\noindent   From a practical point of view,  $\mathrm{(\nameref{ass:dom})}$ has a simple interpretation: it says that if the data-generating distribution $\data$ assigns zero probability to any set $A\in\FZ$, then there must be zero probability of generating a positive sample in $A$, i.e., $M(z,A)=0$, for any point $z\in\Z$.

When $\mathrm{(\nameref{ass:dom})}$ holds, $M(z,\cdot)$ admits a density with respect to $\data$, denoted $\frac{\d M(z,\cdot)}{\d \data} (\cdot)$ and   we may define the empirical Markov kernel and probability measure,
$$
\widehat{M}_m(z,A)\coloneqq \frac{\sum_{i=1}^m \frac{\d M(z,\cdot)}{\d \data} (\xi_i) \delta_{\xi_i}(A)}{\sum_{i=1}^m \frac{\d M(z,\cdot)}{\d \data} (\xi_i)},\qquad \widehat{\pi}_{\mathrm{data},m}(A)\coloneqq \frac{1}{m} \sum_{i=1}^m \delta_{\xi_i}(A),\qquad z\in\Z,\,A\in\FZ.
$$
Here $\xi_1,\ldots,\xi_m$ are the unseen data as in the similarity search \eqref{eq:sim_search}-\eqref{eq:soft_sim_search}. The Markov kernel $\widehat{M}_m$ has the interpretation of importance-weighting each of the points $\xi_1,\ldots,\xi_m$ to account for how likely they are under $M(z,\cdot)$ versus $\data$. We may interpret such weighting as an abstract measure of similarity to the query point $z$, implicitly conveyed by whatever randomised transformations, or ``views'', constitute the Markov kernel $M$.   

Similarly to \eqref{eq:cross_ent_defn}, we define the cross-entropy between 
$\widehat{M}_m$ and $\widehat{Q}_{\tau,m}^\phi$,
\begin{align}
\mathrm{CrossEnt}\left[\widehat{M}_m(z,\cdot)\| \widehat{Q}_{\tau,m}^\phi(z,\cdot)\right]&\coloneqq - \int_{\Z}  \log \left[\frac{\d \widehat{Q}_{\tau,m}^\phi(z,\cdot) }{\d \widehat{\pi}_{\mathrm{data},m}}(z^\prime) \right]\widehat{M}_m(z,\d z^\prime ) \label{eq:cross_ent_emp}\\
& = -\frac{1}{\tau}\frac{\sum_{i=1}^m \frac{\d M(z,\cdot)}{\d \data}(\xi_i) \langle \phi(z),\phi(\xi_i) \rangle}{\sum_{i=1}^m \frac{\d M(z,\cdot)}{\d \data}(\xi_i)}   + \log \left(\frac{1}{m}\sum_{i=1}^m e^{\langle \phi(z),\phi(\xi_i) \rangle/\tau}\right). 
\end{align}

\begin{prop}\label{prop:Q_hat} If $\mathrm{(\nameref{ass:dom})}$ holds and $\xi_1,\xi_2,\ldots\simiid \data$,  then for any $\phi\in \BS$, $k \geq 1$ and $\tau>0$,
\begin{multline*}
\left|\int_{\Z} \mathbb{E}\left(\mathrm{CrossEnt}\left[\widehat{M}_k(z,\cdot)\| \widehat{Q}_{\tau,k}^\phi(z,\cdot)\right]\right) \data(\d z)  - R(\phi;k,\tau)\right|\\
\leq \frac{4}{\tau k}\mathbb{E}\left[\left|\frac{\d M(Z,\cdot)}{\d \data}(Z^\prime)\right|^2\right] + \log\left(1+\frac{e^{2/\tau}}{k}\right),
\end{multline*}
where on the r.h.s., $Z,Z^\prime \simiid \data$.
\end{prop}
Proposition \ref{prop:Q_hat} tells us that when the unseen data points $\xi_1,\xi_2,\ldots$ are drawn from $\data$, and are equal in number to the number of  negative samples per training tuple, $k$, then $R(\phi;k,\tau)$ is $O(1/k)$ close to the expected value of \eqref{eq:cross_ent_emp}, integrated with respect to $\data$. Here the expectation, $\mathbb{E}(\cdot)$ in proposition \ref{prop:Q_hat}, integrates out $\xi_1,\ldots,\xi_k\simiid\data$. In this sense, hypothetically minimising $R(\phi;k,\tau)$ with respect to $\phi$, when $k$ is large, pushes $\widehat{Q}_{\tau,k}^\phi $ towards $\widehat{M}_k$. In turn, this can be interpreted as meaning that the softmax similarity search \eqref{eq:soft_sim_search}, with $m=k$, approximates sampling $\xi_\star \sim \widehat{M}_k(z,\cdot)$.

In practice the number of unseen data $m$ will generally not be equal to the number of negative samples $k$ per training tuple. Rather, we present the case $m=k$ because it is relatively simple to analyse mathematically. One could bound the difference between $R(\phi;k,\tau)$ and the expected cross-entropy with some $m\neq k$ using similar arguments to those in the proofs of propositions \ref{prop:bias_simple} and \ref{prop:Q_hat}, but for brevity we do not pursue such a bound here. In any case, propositions \ref{prop:bias_simple} and \ref{prop:Q_hat} together imply that $\int_{\Z} \mathbb{E}\left(\mathrm{CrossEnt}\left[\widehat{M}_k(z,\cdot)\| \widehat{Q}_{\tau,k}^\phi(z,\cdot)\right]\right) \data(\d z)$ converges to $\int_{\Z}\mathrm{CrossEnt}[M(z,\cdot)\| Q_\tau^\phi(z,\cdot)]\data(\d z)$ as $k\to\infty$.

\subsection{Simplified bounds for the DCL loss function}\label{sec:simplified_bounds}

Propositions \ref{prop:bias_simple} and \ref{prop:Q_hat} relate $R(\phi;k,\tau)$ to two different quantities.
Can we say anything about whether, for finite $k$, $R(\phi;k,\tau)$ is closer to one or the other? The picture becomes clearer if we consider a slight change to the InfoNCE loss. Suppose the term ``$1+$'' is omitted from the loss in algorithm \ref{alg:InfoNCE}. This was called the Decoupled Contrastive Learning (DCL) loss by \citet{yeh2022decoupled}, who discussed its properties and gave evidence of superior performance in practice.  The population risk becomes:
$$
\widetilde{R}(\phi;k,\tau)\coloneqq 
\mathbb{E}\left[ \log \left(  \sum_{j=1}^k e^{\langle\phi(Z^a),\phi(Z_{j}^-) - \phi(Z^+)\rangle/\tau}\right)\right] - \log k,$$
instead of \eqref{eq:R_k_defn}. 

It can be checked that with fairly minor modifications to the proofs of propositions \ref{prop:bias_simple} and  \ref{prop:Q_hat} (see appendix \ref{app:calcs_for_DCL}), under all the same conditions as in those propositions, the DCL population risk $\widetilde{R}(\phi;k,\tau)$ satisfies:
\begin{equation}\label{eq:R_tilde_prop1}
0\leq \int_{\Z}\mathrm{CrossEnt}[M(z,\cdot)\| Q_\tau^\phi(z,\cdot)]\data(\d z)- \widetilde{R}(\phi;k,\tau)  \leq  \frac{1}{8k}(e^{2/\tau }-1)^2,
\end{equation}
i.e., compared to the proposition \ref{prop:bias_simple} the lower bound is zero, and
\begin{equation}\label{eq:R_tilde_prop2}
\left|\int_{\Z} \mathbb{E}\left(\mathrm{CrossEnt}\left[\widehat{M}_k(z,\cdot)\| \widehat{Q}_{\tau,k}^\phi(z,\cdot)\right]\right) \data(\d z)  - \widetilde{R}(\phi;k,\tau)\right|
\leq \frac{4}{\tau k}\mathbb{E}\left[\left|\frac{\d M(Z,\cdot)}{\d \data}(Z^\prime)\right|^2\right],
\end{equation}
i.e., the additive term  $\log(1+e^{2/\tau}/k)$ in proposition \ref{prop:Q_hat} vanishes. 

We can see that if $\tau\to0$, $k$ must grow exponentially fast in $1/\tau$ to control the bound in \eqref{eq:R_tilde_prop1}, but only as fast as $1/\tau$ to control the bound in \eqref{eq:R_tilde_prop2}. In practice,  $\tau$ is often chosen somewhere in the range $0.07-0.5$ \citep{chen2020simclr,he2020moco}, so $e^{1/\tau}$ could be quite a large number. This prompts the question: is the bound in \eqref{eq:R_tilde_prop1} tight? The answer is that for small values of $k$ it generally is not tight (lemma \ref{lem:determ_bound_R_Tilde} shows that an alternative  bound $2/\tau$ holds for any $k\geq 1$ and $\tau>0$), but as $k$ grows, exponential dependence on $e^{1/\tau}$ cannot be avoided: it is shown in appendix \ref{app:calcs_for_DCL} that if $e^{1/\tau} >2$, then  an example can be constructed such that
\begin{multline}\label{eq:ratio_bound}
\limsup_{k\to\infty} \frac{\left|\int_{\Z} \mathbb{E}\left(\mathrm{CrossEnt}\left[\widehat{M}_k(z,\cdot)\| \widehat{Q}_{\tau,k}^\phi(z,\cdot)\right]\right) \data(\d z)  - \widetilde{R}(\phi;k,\tau)\right|}{\int_{\Z}\mathrm{CrossEnt}[M(z,\cdot)\| Q_\tau^\phi(z,\cdot)]\data(\d z)- \widetilde{R}(\phi;k,\tau)} \\
\leq \frac{65}{\tau(e^{1/\tau}-2) }\mathbb{E}\left[\left|\frac{\d M(Z,\cdot)}{\d \data}(Z^\prime)\right|^2\right].
\end{multline}
The numerical constant $65$ represents some particular choices made in the construction of this example and so may be improved. Therefore, in this example and for large $k$, the numerator on the left hand side of  \eqref{eq:ratio_bound} will be much smaller than the denominator if
$$
\tau e^{1/\tau} \gg \mathbb{E}\left[\left|\frac{\d M(Z,\cdot)}{\d \data}(Z^\prime)\right|^2\right].
$$

\section{Generalisation analysis}\label{sec:generalisation}

The results of  section \ref{sec:sim_search} explain the behaviour of the population risk $R(\phi;k,\tau)$ when $k$ is large, with the bounds in propositions \ref{prop:bias_simple} and \ref{prop:Q_hat} holding uniformly over any embedding function $\phi\in \BS$. As such, those bounds apply to $R(\widehat{\phi},k,\tau)$ where the embedding function $\widehat{\phi}$ is  obtained by empirical risk minimisation as in algorithm \ref{alg:InfoNCE}, or by any other approximate minimisation scheme which outputs some member of $\BS$. 

In order to quantify generalisation performance we would like to upper bound $R(\widehat{\phi},k,\tau)$ in terms of  $\widehat{R}_n(\widehat{\phi},k,\tau)$. This would tell us how the quality of training, i.e., achieving a small value of $\widehat{R}_n(\widehat{\phi},k,\tau)$, transfers to similarity search on unseen data, as discussed in section \ref{sec:sim_search}. Following the usual workflow of statistical learning theory, instead of analysing $\widehat{R}_n(\widehat{\phi},k,\tau)$ directly, we seek to upper bound $R(\cdot,k,\tau)$ in terms of $\widehat{R}_n(\cdot,k,\tau)$ uniformly over some class of embedding functions  $\Phi\subset \BS$ to which $\widehat{\phi}$ is supposed to belong. Section \ref{sec:gen_and_cont} sets out tools for doing so applicable to a general function class $\Phi$. These tools are applied in section  \ref{sec:app_to_lipschitz_functions} to a specific class of embedding functions which are Lipschitz in a parameter.

\subsection{Generalisation and continuity bounds}\label{sec:gen_and_cont}

For $\mathbf{z}_i=(z_i^a,z_i^+,z^-_{i1},\ldots,z_{ik}^-)\in\Z^{2+k}$ define:
\begin{equation}\label{eq:ell_defn}
\ell(\phi,\mathbf{z}_i,k,\tau) \coloneqq  \log \left(\frac{1}{k}+\frac{1}{k}\sum_{j=1}^k e^{\langle\phi(z_i^a),\phi(z_{ij}^-) - \phi(z_i^+)\rangle/\tau} \right)
\end{equation}
so that substituting the random tuple $\mathbf{Z}_i$ in place of $\mathbf{z}_i$ we have:
$$
\widehat{R}_n(\phi;k,\tau) = \frac{1}{n}\sum_{i=1}^n \ell(\phi,\mathbf{Z}_i,k,\tau).
$$
The following proposition is an application to the InfoNCE risk of a well-known empirical Rademacher complexity generalisation bound for additive loss functions  \citep[Thm. 3.3]{mohri2018foundations}.
\begin{prop}\label{prop:uniform} Let $\Phi$ be a subset of $\BS$.  For any $n\geq1$, $k\geq 1$, $\tau>0$ and $\delta\in(0,1)$, it holds with probability at least $1-\delta$ that for all $\phi\in\Phi$,
$$
R(\phi;k,\tau)\leq \widehat{R}_n(\phi;k,\tau)+ 2 \mathbb{E}\left[\left.\sup_{\phi\in\Phi}\frac{1}{n}\sum_{i=1}^n\sigma_i \ell(\phi,\mathbf{Z}_i,k,\tau)\right|\mathbf{Z}_1,\ldots,\mathbf{Z}_n\right] + \frac{12}{\tau}\sqrt{\frac{\log \frac{2}{\delta}}{2n}},
$$
where $\sigma_1,\ldots,\sigma_n$ are i.i.d. Rademacher variables, independent of $\mathbf{Z}_1,\ldots,\,\mathbf{Z}_n$.
\end{prop}
The conditional expectation term in proposition \ref{prop:uniform} is the empirical Rademacher complexity associated with the InfoNCE loss functional \eqref{eq:ell_defn} evaluated over $\Phi$, with respect to the training sample of tuples $\mathbf{Z}_1,\ldots,\mathbf{Z}_n$. This empirical Rademacher complexity quantifies generalisation error, i.e., the capacity of $\Phi$, via \eqref{eq:ell_defn}, to overfit the training data. 

To obtain insight into this generalisation error, we need to bound the empirical Rademacher complexity in such a way that the roles of $\Phi$, $k$, $\tau$, $n$, etc., become clear. A crucial step towards such a bound  is to obtain some kind of quantitative continuity estimate for the loss function with respect to the embedding function $\phi$. Such an estimate opens the door to bounds on Rademacher complexity using well-known techniques such as contraction \citep{maurer2016vector}, and/or covering numbers, chaining and Dudley's integral lemma, e.g., \citep[Ch. 5]{wainwright2019high}. The following proposition is the main technical contribution of section \ref{sec:generalisation}. Here $\beta\geq 1$ is a mathematical parameter, introduced in order to be able to tune the bound in  proposition \ref{prop:ell_lipschitz} to account for the algorithmic temperature parameter $\tau$ (this tuning is demonstrated in section \ref{sec:app_to_lipschitz_functions}). As such, we call $\beta$ the ``inverse temperature'' parameter.

\begin{prop}\label{prop:ell_lipschitz} For any $k\geq 1$, $\mathbf{z}=(z^a,z^+,z^-_1,\ldots,z_k^-)\in\Z^{2+k}$, $\tau>0$, $\beta\geq 1$, and  $\phi,
\phi^\prime \in \BS$,
\begin{multline*}
\left|\ell(\phi,\mathbf{z},k,\tau) - \ell(\phi^\prime,\mathbf{z},k,\tau)  \right| \\
\leq \frac{4}{\tau}\|\phi(z^a)-\phi^\prime(z^a)\|_2\nonumber  +\frac{1}{\tau}\|\phi(z^+)-\phi^\prime(z^+)\|_2 \nonumber  +\frac{3}{\tau}  e^{2/{\beta\tau}}\left(\frac{1}{k}\sum_{j=1}^k \|\phi(z_j^-)-\phi^\prime(z_j^-)\|_2^{\beta}\right)^{1/\beta}. 
\end{multline*}
\end{prop}
A key feature of the bound in proposition \ref{prop:ell_lipschitz} is that it preserves the structure of averaging over the negative samples present in the InfoNCE loss \eqref{eq:ell_defn}; specifically the negative samples $z_1^-,\ldots,z_k^-$ enter into the bound   only through a ``power-mean'' parameterised by the inverse temperature $\beta$. We shall see in section \ref{sec:app_to_lipschitz_functions} how this averaging over negative samples transfers to bounds on the Rademacher complexity.

Let us compare the bound in proposition  \ref{prop:ell_lipschitz} to an alternative estimate available in the literature, derived from Lipschitz continuity of the   logistic loss:
$$
\ell^\mathrm{log}(v) \coloneqq  \log\left(1+\sum_{j=1}^k \exp(-v_j)\right),\qquad v = (v_1,\ldots,v_k)\in\R^k,
$$
as discussed in, e.g., \citep{lei2023generalization}. Note that when:
\begin{equation}\label{eq:v_j}
v_j =\langle\phi(z^a),\phi(z^+)-\phi(z_j^-)\rangle/\tau,
\end{equation}
we have $\ell^{\mathrm{log}}(v)-\log k = \ell(\phi,\mathbf{z},k,\tau)$.

It is known that $\ell^{\mathrm{log}}$ is $1$-Lipschitz with respect to the $\|\cdot\|_\infty$ norm on $\mathbb{R}^k$ \citep{lei2019data} (and hence also $1$-Lipschitz with respect to the $\|\cdot\|_p$ norm for any $p\geq 1$), that is:
\begin{equation}\label{eq:logistic_lipschitz}
|\ell^{\mathrm{log}}(v)-\ell^{\mathrm{log}}(v^\prime)| \leq \|v-v^\prime\|_{\infty},
\end{equation}
for all $v,v^\prime\in\R^k$. In the case \eqref{eq:v_j} note that $\ell^{\mathrm{log}}(v)$ and $\ell(\phi,\mathbf{z},k,\tau)$ differ only by an additive factor of $\log k$. Therefore, the $1$-Lipschitz property of $\ell^{\mathrm{log}}(\cdot)$ transfers to $\ell(\cdot,\mathbf{z},k,\tau)$ as follows: for any $\phi,\phi^\prime\in \BS$ let \eqref{eq:v_j} hold and let $v^\prime=(v_1^\prime,\ldots,v_k^\prime)$ be defined by replacing $\phi$   in \eqref{eq:v_j} with $\phi^\prime$. Then as a consequence of \eqref{eq:logistic_lipschitz} we have:
\begin{align}
&\left|\ell(\phi,\mathbf{z},k,\tau) - \ell(\phi^\prime,\mathbf{z},k,\tau)  \right|\nonumber \\
& \leq \frac{1}{\tau}\max_{1\leq j\leq k} \left|\langle\phi(z^a),\phi(z^+)-\phi(z_j^-)\rangle-\langle\phi^\prime(z^a),\phi^\prime(z^+)-\phi^\prime(z_j^-)\rangle\right| \label{eq:logist_estimate1} \\
&\leq \frac{\sqrt{6}}{\tau}\left(\|\phi(z^a)-\phi^\prime(z^a)\|_2^2 +\|\phi(z^+)-\phi^\prime(z^+)\|_2^2+\max_{1\leq j\leq k} \|\phi(z_j^-)-\phi^\prime(z_j^-)\|_2^2\right)^{1/2},\label{eq:logist_estimate2}
\end{align}
where the second inequality uses the fact established in \citep{lei2023generalization}[Proof of Lemma 4.3] that for any $u^a,u^+,u^-\in\Sd$,  the mapping $(u^a,u^+,u^-)\mapsto\langle u^a,u^+-u^-\rangle$ is $\sqrt{6}$-Lipschitz with respect to the $\|\cdot\|_2$ norm on $\R^{3d}$. 

Crucially in \eqref{eq:logist_estimate1} the structure of averaging over negative samples in  \eqref{eq:ell_defn} has been lost. Comparing to proposition \ref{prop:ell_lipschitz}, we may view \eqref{eq:logist_estimate1} as resorting to maximisation rather than averaging. Noting that in proposition \ref{prop:ell_lipschitz}  we are free to choose the inverse temperature parameter $\beta\geq 1$, we can take $\beta\to\infty$ there, in which case $e^{2/\beta\tau}\to 1$ and the power-mean $(k^{-1}\sum_{j=1}^k \|\phi(z_j^-)-\phi^\prime(z_j^-)\|_2^\beta)^{1/\beta}$ tends to $\max_{1\leq j\leq k} \|\phi(z_j^-)-\phi^\prime(z_j^-)\|_2$. In that limit it can be seen that the bounding quantity in proposition \eqref{prop:ell_lipschitz} and  the r.h.s. of \eqref{eq:logist_estimate2} are equivalent up to a fairly modest numerical scaling due to the elementary inequalities for scalars:  $3^{-1/2} \sum_{i=1}^3 |a_i|\leq (\sum_{i=1}^3 |a_i|^2)^{1/2}\leq  \sum_{i=1}^3 |a_i|$.  The knock-on effect of \eqref{eq:logist_estimate2} for bounding Rademacher complexity is illustrated in section \ref{sec:app_to_lipschitz_functions}.

\subsection{Application to Lipschitz embedding functions}\label{sec:app_to_lipschitz_functions}

To demonstrate how the bound in proposition \ref{prop:ell_lipschitz} can be applied, let us a consider a situation in which the class of functions $\Phi$ is of the form $\{\phi_{\theta};\theta\in\Theta\}\subset \BS$ for some parameter set $\Theta$, and $\phi_\theta$ is Lipschitz with respect to $\theta$ in the sense of the following assumption.
\begin{namedass}[C]\label{ass:Theta} The input space $\Z$ is $\R^{d_{\mathrm{in}}}$ for some $d_\mathrm{in}\geq 1$, and the parameter space is:
\begin{equation}\label{eq:Theta_defn}\Theta=\{\theta\in\R^{d_{\Theta}}:\|\theta - \theta_0\|_2\leq R\},\end{equation}
for some $d_{\Theta}\geq 1$, $\theta_0\in\R^{d_{\Theta}} $ and $R>0$. There is some finite constant $C_{\Phi}$ such that for all $\theta,\theta^\prime \in\Theta$ and $z\in \Z$,
$$
\|\phi_{\theta}(z)-\phi_{\theta^\prime}(z)\|_2\leq C_{\Phi} \|z\|_2\|\theta - \theta^\prime\|_2.
$$
\end{namedass}
The setting of assumption $\mathrm{(\nameref{ass:Theta})}$ could easily be generalised in a number of ways. We consider a finite-dimensional Euclidean parameter $\theta$ and dependence on $\|z\|_2$ to simplify the exposition which follows. There, our priority is to give a swift and easily interpretable demonstration of how $k$, $\tau$, $n$, etc. impact the empirical Rademacher complexity appearing in proposition \ref{prop:uniform}.

We shall use the following pseudo-metric on $\Phi$, associated with a realisation of tuples $\mathbf{z}_1,\ldots,\mathbf{z}_n \in\Z^{2+k}$, to bound Rademacher complexity.
$$
\rho_n^{\mathrm{info}}(\phi,\phi^\prime)\coloneqq \left(\frac{1}{n}\sum_{i=1}^n \left|\ell(\phi,\mathbf{z}_i,k,\tau) - \ell(\phi^\prime,\mathbf{z}_i,k,\tau) \right|^2\right)^{1/2}
$$
The following lemma illustrates how this pseudo-metric can be bounded in the setting of assumption $\mathrm{(\nameref{ass:Theta})}$. The proof of lemma  \ref{lem:metric} uses proposition \ref{prop:ell_lipschitz} with a particular choice of inverse temperature parameter, $\beta=2/\tau$. As noted earlier, in practice $\tau$ is usually chosen somewhere in the range $0.07-0.5$ and  choosing $\beta=2/\tau$ conveniently reduces the factor $e^{2/\beta\tau}$ in proposition \ref{prop:ell_lipschitz} to $e$.
\begin{lem}\label{lem:metric}
If $\mathrm{(\nameref{ass:Theta})}$ holds, then for any $n\geq1$, $k\geq 1$, $\tau > 0$, $\mathbf{z}_1,\ldots,\mathbf{z}_n\in\Z^{2+k}$ where $\mathbf{z}_i=(z_i^a,z_i^+,z_{i1}^-,\ldots,z_{ik}^-)$, and $\theta,\theta^\prime\in\Theta$,
\begin{equation*}
\rho_n^{\mathrm{info}}(\phi_{\theta},\phi_{\theta^\prime})\leq\frac{4 e}{\tau}  B_\tau(\mathbf{z_1},\ldots,\mathbf{z}_n)C_{\Phi}  \|\theta-\theta^\prime\|_2  , 
\end{equation*}
where $C_\Phi$ is as in assumption $\mathrm{(\nameref{ass:Theta})}$  and
$$
B_\tau(\mathbf{z}_1,\ldots,\mathbf{z}_n) \coloneqq \left(\frac{1}{n}\sum_{i=1}^n\|z_i^a\|_2^2\right)^{1/2} + \left(\frac{1}{n}\sum_{i=1}^n\|z_i^+\|_2^2\right)^{1/2} + \left(\frac{1}{nk}\sum_{i=1}^n\sum_{j=1}^k\|z_{ij}^-\|_2^{2/(1\wedge\tau)}\right)^{(1\wedge\tau)/2}.
$$
\end{lem}
Lemma \ref{lem:metric} is put to use along with Dudley's entropy integral in the proof of the following proposition, which bounds the empirical Rademacher complexity.

\begin{prop}\label{prop:prop_dudley_apply} Assume $\mathrm{(\nameref{ass:Theta})}$ holds and let $C_{\Phi}$,  $d_{\Theta}$  and $R$ be as therein. For any  $n\geq1$, $k\geq1$, $\mathbf{z}_1,\ldots,\mathbf{z}_n \in\Z^{2+k}$  where $\mathbf{z}_i = (z_i^a,z_i^+,z_{i1}^-,\ldots,z_{ik}^-)$ and $\tau>0$,   let $B_{\tau}(\mathbf{z}_1,\ldots,\mathbf{z}_n)$  be as in assumption $\mathrm{(\nameref{ass:Theta})}$. Then,
$$
\mathbb{E}\left[ \sup_{\phi\in\Phi}\frac{1}{n}\sum_{i=1}^n\sigma_i \ell(\phi,\mathbf{z}_i,k,\tau)  \right] \leq \frac{96}{\tau}\sqrt{\frac{d_{\Theta}}{n}}\min\left\{a\sqrt{\log(3)+1},\sqrt{\log(1+2a)+1}\right\},
$$
where $a\coloneqq eR B_{\tau}(\mathbf{z}_1,\ldots,\mathbf{z}_n)C_{\Phi} $ and $\sigma_1,\ldots,\sigma_n$ are i.i.d. Rademacher variables.
\end{prop}
We make the following observations on the bound in proposition \ref{prop:prop_dudley_apply}
\begin{itemize}
\item On the right of the inequality, the only place that $k$ appears is in the averaging term within $B_{\tau}(\mathbf{z}_1,\ldots,\mathbf{z}_n)$. If the random sample $\mathbf{Z}_1,\ldots,\mathbf{Z}_n$ was substituted in place of $\mathbf{z}_1,\ldots,\mathbf{z}_n$, then under the i.i.d. property of $Z_{ij}^-$ in assumption $\mathrm{(\nameref{ass:dist2})}$ the law of large numbers would be applicable to this averaging term. In that sense we see the Rademacher complexity stabilises as $k\to\infty$. It can be checked that using \eqref{eq:logist_estimate2} instead of proposition \ref{prop:ell_lipschitz} results in a maximisation, rather than average, over the negative sample norms.
\item Other than the leading factor of $\tau^{-1}$, the influence  of the temperature parameter has been transferred all the way through to the inverse exponent of a power-mean of the norms of negative samples $\|z_{ij}^-\|_2$, within $B_{\tau}(\mathbf{z}_1,\ldots,\mathbf{z}_n)$. Thus as the temperature tends to zero $\tau\to0$, the bound becomes increasingly sensitive to large values of the negative sample norms $\|z_{ij}^-\|_2$, suggesting that, roughly speaking, generalisation error is bigger at lower temperatures.
\item The embedding dimension, $d$ in $\Sd$, makes no appearance in the bound of proposition \ref{prop:prop_dudley_apply}. This lack of dimension dependence can be understood as an advantage of embedding onto the sphere $\Sd$, rather than, e.g., embedding in $\mathbb{R}^d$  in an unconstrained manner.
\item To interpret the minimum term, note that $a\leq 1 \Leftrightarrow a\sqrt{\log(3)+1}\leq \sqrt{\log(1+2a)+1}$. Therefore as $a\to 0$, the bound is $O(a)$, whilst as $a\to\infty$ it is $O(\sqrt{\log a})$. In turn, if either the parameter radius $R$, data-dependent term $B_{\tau}(\mathbf{z}_1,\ldots,\mathbf{z}_n)$ or Lipschitz constant $C_{\Phi}$ were to tend to zero (with all other quantities held constant), then so does $a$, and hence the Rademacher complexity shrinks to zero.
\end{itemize}

\section{Discussion}\label{sec:discussion}

\paragraph{A combined similarity search generalisation bound.} To summarise some of the results of this work we can combine, for example, propositions \ref{prop:Q_hat}, \ref{prop:uniform} and \ref{prop:prop_dudley_apply} in the following theorem.
\begin{thm}\label{thm:combined} Assume $\mathrm{(\nameref{ass:dist2})}$,  $\mathrm{(\nameref{ass:dom})}$ and $\mathrm{(\nameref{ass:Theta})}$ with $\Phi=\{\phi_\theta;\theta\in\Theta\}$ as therein and $\xi_1,\xi_2,\ldots \simiid \data$. For any $n\geq 1$, $k\geq 1$, $\tau>0$ and $\delta\in(0,1)$ it holds with probability at least $1-\delta$ that for all $\phi\in\Phi$,
\begin{align*}
\underbrace{\int_{\Z} \mathbb{E}\left(\mathrm{CrossEnt}\left[\widehat{M}_k(z,\cdot)\| \widehat{Q}_{\tau,k}^\phi(z,\cdot)\right]\right) \data(\d z)}_{\mathrm{similarity\;search\;expected\;cross-entropy}} 
 &
 \leq \underbrace{\widehat{R}_n(\phi;k,\tau)}_{\mathrm{empirical\;risk}} \\
&
 +\underbrace{ \frac{\mathrm{const.}}{\tau}\sqrt{\frac{d_{\Theta}}{n}}\min\left\{A\sqrt{\log(3)+1},\sqrt{\log(1+2A)+1}\right\}}_{\mathrm{complexity\;penalty}}\\
& + \underbrace{\frac{\mathrm{const.}}{\tau}\sqrt{\frac{\log \frac{2}{\delta}}{2n}}}_{\mathrm{sample \;variability}} \\
&+ \underbrace{\frac{4}{\tau k}\mathbb{E}\left[\left|\frac{\d M(Z,\cdot)}{\d \data}(Z^\prime)\right|^2\right] + \log\left(1+\frac{e^{2/\tau}}{k}\right)}_{\mathrm{finite\;} k \mathrm{\;bias}},
\end{align*}
where $A\coloneqq eR B_{\tau}(\mathbf{Z}_1,\ldots,\mathbf{Z}_n)C_{\Phi}$; $d_{\Theta}$, $R$ and $C_\Phi$ are as in $\mathrm{(\nameref{ass:Theta})}$; and  $B_{\tau}(\cdots)$ is as in lemma \ref{lem:metric}.   
\end{thm}
This theorem illustrates the three contributions to the difference between similarity search expected cross entropy and the empirical risk $\widehat{R}_n(\phi,k,\tau)$. The sample variability and complexity penalty terms arise from proposition \ref{prop:uniform} combined with the Rademacher complexity bound in proposition \ref{prop:prop_dudley_apply}. The finite $k$ bias term comes from proposition \ref{prop:Q_hat}. The sample variability term does not depend on $k$ and is $O(n^{-1/2})$, the complexity penalty term stabilises as $k\to\infty$ because of the averaging structure within $B_{\tau}(\mathbf{Z}_1,\ldots,\mathbf{Z}_n)$ and hence is $O(n^{-1/2})$, and the bias term tends to zero as $k\to\infty$ and is $O(1/k)$. As discussed in section \ref{sec:simplified_bounds}, if using the DCL loss instead of InfoNCE, the $\log(1+e^{2/\tau}/k)$ term disappears from the finite $k$ bias. We leave it as an exercise to check that counterparts of proposition \ref{prop:uniform} and proposition \ref{prop:prop_dudley_apply} hold for the DCL risk, yielding complexity penalty and sample variability terms equal to those in theorem \ref{thm:combined} up to numerical constants.

\paragraph{Regularisation from early stopping of optimisation algorithms.} Our analysis has been algorithm-free, in the sense that we have  not considered any particular method for approximately minimising the InfoNCE risk. Never-the-less, if $\theta_0$ in $\mathrm{(\nameref{ass:Theta})}$ is regarded as an initial point for a recursive optimisation algorithm  (which optimises $\theta$ in order to approximately minimise the empirical risk associated with $\phi_\theta$) then $R$ can be interpreted as a bound on the distance from $\theta_0$ such an algorithm can move in some finite number of steps. We can see in proposition \ref{prop:prop_dudley_apply} that $R\to 0$ forces the Rademacher complexity to zero. In this sense, our results  accommodate the idea that early stopping of an optimisation algorithm has a regularising effect.

\paragraph{Neural networks.} In our analysis we have prioritised obtaining a clear and easily interpretable  demonstration of the impact of $k$, $\tau$, etc. on generalisation. Assumption $\mathrm{(\nameref{ass:Theta})}$ serves this purpose. In fact, assumption $\mathrm{(\nameref{ass:Theta})}$ is satisfied if $\phi_{\theta}$ is a fully connected multilayer perceptron with: Euclidean input domain; a $1$-Lipschitz activation function such as ReLU; zero biases; weights constrained such that the product of spectral norms of the weight matrices is less than a constant; and $\theta$ comprises all entries of all the weight matrices of the network. In this situation $d_\Theta \sim \text{width}^2\times\text{depth}$. The resulting factor of $d_{\Theta}$ in the bounds of proposition \ref{prop:prop_dudley_apply} could be ameliorated to some extent by scaling down the weight matrices (and hence $C_{\Phi}$) with network size, as is common practice, e.g., \citep{yang2021tuning}. However, for wide and/or deep networks, the explicit dependence on $d_{\Theta}$ could make the bound in proposition \ref{prop:prop_dudley_apply} vacuous unless $n$ is very large. Various techniques have been devised to obtain complexity bounds for neural networks which do not have explicit dependence on width and/or depth, e.g., \citep{neyshabur2015norm,bartlett2017spectrally,golowich2018size,lei2023generalization}. A potential topic for future research is to explore whether such techniques can be made to work together with the continuity bound in proposition \ref{prop:ell_lipschitz}. Another exciting direction is the notion of a \emph{path metric} of a neural network, e.g., \citep{gonon2025rescaling}, which can help to bound complexity for a broad class of modern neural network architectures beyond the basic format of a multilayer perceptron and achieve scaling invariance.

\paragraph{Downstream classification.} Whilst we have emphasised similarity search generalisation, the results of section \ref{sec:generalisation} do not rely on any distributional assumptions other than the tuples $\mathbf{Z}_1,\ldots,\mathbf{Z}_n$ being i.i.d. As such, they are transferable to the classification setting of \cite{saunshi2019theoretical,lei2023generalization}.

\newpage
\appendix

\section{Proofs and supporting results for section \ref{sec:contrastive_learning}}

\begin{lem}\label{lem:bias_taylor}
For any $k\geq1$, $b,c>0$ and i.i.d. random variables $X$, $X_1,\ldots,X_k$, each valued in the interval $[-c,c]$,
$$
-\log\left(1+\frac{e^c}{bk}\right) \leq  \log\left( b\mathbb{E}\left[ e^X\right]\right) - \mathbb{E}\left[\log \left(\frac{1}{k}+\frac{b}{k} \sum_{j=1}^k e^{X_j}\right)\right] \leq \frac{1}{8k} (e^{2c}-1)^2. 
$$
\end{lem}
\begin{proof}
Denote $\widehat{\mu}\coloneqq k^{-1} \sum_{j=1}^k e^{X_j}$ and $\mu \coloneqq \mathbb{E}[e^X]$. Note $\widehat{\mu},\mu  \in[e^{-c},e^c]$.  With $f(y)\coloneqq\log(1/k +y)$ we have:
\begin{equation}\log \left( b\mathbb{E}\left[ e^X\right]\right) - \mathbb{E}\left[\log \left(\frac{1}{k}+\frac{b}{k} \sum_{j=1}^k e^{X_j}\right)\right] = \log (b\mu) - f(b\mu) +f(b\mu) - \mathbb{E}[f(b\widehat{\mu})]. \label{eq:bias_add_takeaway} \end{equation}
For the first difference on the r.h.s. of \eqref{eq:bias_add_takeaway},
\begin{equation}\label{eq:log_bmu_bound}
0\geq \log (b\mu) - f(b\mu) = \log (b\mu) - \log\left[b\mu\left(1+\frac{1}{kb\mu}\right)\right] \geq - \log\left(1+\frac{e^c}{bk}\right).
\end{equation}

For the second difference on the r.h.s. of \eqref{eq:bias_add_takeaway}, 
by Taylor expansion of $y\mapsto f(y)$ about $b\mu$, there exists $\xi \in [be^{-c}, be^c]$ such that:
\begin{align*}
\log\left(\frac{1}{k}+b\widehat{\mu}\right) & =\log\left(\frac{1}{k}+b\mu\right)+\frac{b(\widehat{\mu}-\mu)}{\frac{1}{k}+b\mu}-\frac{b^2}{2(\frac{1}{k}+\xi)^{2}}\left(\widehat{\mu}-\mu\right)^{2}\\
 & \geq \log\left(\frac{1}{k}+b\mu\right)+\frac{b( 
 \widehat{\mu}-\mu)}{\frac{1}{k}+b\mu}-\frac{1}{2 e^{-2c}}\left(
 \widehat{\mu}-\mu\right)^{2}.
\end{align*}
Therefore, using $\mathbb{E}[\widehat{\mu}]=\mu$,
\[
\mathbb{E}\left[f(b\widehat{\mu})\right]\geq f(b\mu)-\frac{1}{2 e^{-2c}}\mathbb{E}\left[\left(\widehat{\mu}-\mu\right)^{2}\right].
\]
Using Popoviciu's inequality on variances,
\[
\mathbb{E}\left[\left(\widehat{\mu}-\mu\right)^{2}\right]=\frac{1}{k}\mathrm{Var}\left[e^{X}\right]\leq\frac{1}{4k}\left(e^{c}-e^{-c}\right)^{2}.
\]
Concavity of $f$ and Jensen's inequality implies $\mathbb{E}[f(b\widehat{\mu})] \leq f(b\mu)$, so we have: 
\begin{align}
0\leq f(b\mu) - \mathbb{E}\left[f(b\widehat{\mu})\right] & \leq \frac{1}{8 k}\left(e^{2c}-1\right)^{2}.\label{eq:bmu_bound}
\end{align}
The proof is completed by combining \eqref{eq:log_bmu_bound} and \eqref{eq:bmu_bound} with \eqref{eq:bias_add_takeaway}.
\end{proof}

\begin{proof}[Proof of proposition \ref{prop:bias_simple}]

Fix any $z^a,z^+ \in\Z$. Let $Z_1^-,\ldots, Z_k^-\simiid \data$. Applying lemma \ref{lem:bias_taylor} with $b=e^{-\langle\phi(z^a),\phi(z^+)\rangle/\tau}$, $c= 1/\tau$, $X_j=e^{\langle \phi(z^a),\phi(Z_j^-)\rangle/\tau }$, and using $b\geq e^{-1/\tau}$,
\begin{align*}
& -\log\left(1+\frac{e^{2/\tau}}{k}\right) \\
&\leq  \log\left(\int_{\Z}e^{\langle\phi(z^a),\phi(z^-)-\phi(z^+)\rangle/\tau}\data(\d z^-)\right) - \mathbb{E}\left[\log \left(\frac{1}{k}+\frac{1}{k} \sum_{j=1}^k e^{\langle\phi(z^a),\phi(Z_j^-)-\phi(z^+)\rangle/\tau}\right)\right]\\
&\leq \frac{1}{8k} (e^{2/\tau}-1)^2.
\end{align*}
Now let  $Z^a\sim \data$ and $Z^+\sim M(Z^a,\cdot)$,  independent of $Z_1^-,\ldots,Z_k^-$. 
The proof is completed by  substituting $Z^a$ and $Z^+$ in place of $z^a$ and $z^+$ respectively, and taking expectation.
\end{proof}

\begin{lem}\label{lem:self_norm_is} For any two probability measures $\nu,\pi \in\P(\Z)$ such that $\nu \ll \pi$, any measurable function $f:\Z\to[-1,1]$, any $m\geq 1$ and $\xi_1,\ldots,\xi_m\simiid \pi$,
$$
\left| \mathbb{E}\left[\frac{\sum_{j=1}^m f(\xi_j)\frac{\d\nu}{\d\pi}(\xi_j)}{\sum_{j=1}^m \frac{\d\nu}{\d\pi}(\xi_j)}\right] - \mathbb{E}_{\nu}[f(\xi)]\right|\leq \frac{4}{m}
\mathbb{E}_{\pi}\left[\left|\frac{\d\nu}{\d\pi}(\xi)\right|^2\right],
$$
where $\mathbb{E}_{\nu}$ and $\mathbb{E}_{\pi}$ denote expectation with respectively $\xi\sim \nu$ and $\xi\sim \pi$.
\end{lem}
\begin{proof} 
Define  $A\coloneqq m^{-1} \sum_{j=1}^m \frac{\d\nu}{\d\pi}(\xi_j)\left(f(\xi_j) -\mathbb{E}_{\nu}[f(\xi)]\right) $ and $B\coloneqq m^{-1} \sum_{j=1}^m \frac{\d\nu}{\d\pi}(\xi_j)$.  Then 
$$
\frac{\sum_{j=1}^m f(\xi_j)\frac{\d\nu}{\d\pi}(\xi_j)}{\sum_{j=1}^m \frac{\d\nu}{\d\pi}(\xi_j)} - \mathbb{E}_{\nu}[f(\xi)] = \frac{A}{B},
$$
and the quantity we seek to bound is $|\mathbb{E}[A/B]|$.

We have:
$$
\frac{1}{B} = 2- B +\frac{(B-1)^2}{B}, 
$$
and using the fact that $\mathbb{E}[A]=0$,
\begin{equation}\label{eq:E[A/B]_bound}
\mathbb{E}\left[\frac{A}{B}\right] = -\mathbb{E}[AB] +\mathbb{E}\left[\frac{A}{B}(B-1)^2\right].
\end{equation}
Using the facts that $\xi_1,\ldots\xi_m\simiid\pi$ and $\mathbb{E}\left[\frac{\d \nu }{\d \pi}(\xi_j) f(\xi_j)\right] = \mathbb{E}_{\nu}[f(\xi)]$, the expectation of every cross term in the product-of-sums $AB$ is zero, so we have:
$$
\mathbb{E}[AB] = \mathbb{E}\left[\frac{1}{m^2}\sum_{j=1}^m \left|\frac{\d \nu }{\d \pi}(\xi_j) \right|^2 \left(f(\xi_j) - \mathbb{E}_{\nu}[f(\xi)] \right) \right],
$$
and since $|f(z)|\leq 1$ for all $z\in\Z$, we obtain:
\begin{equation}\label{eq:E[AB]_bound}
|\mathbb{E}[AB] | \leq \frac{2}{m} \mathbb{E}_{\pi}\left[\left|\frac{\d \nu}{\d \pi}(\xi)\right|^2\right].
\end{equation}

Again using $|f(z)|\leq 1$ we have $|A|\leq 2B$,
hence: 
\begin{equation}\label{eq:E[B^2]_bound}
\left|\mathbb{E}\left[\frac{A}{B}(B-1)^2\right]\right|\leq 2 \mathbb{E}[(B-1)^2] = \frac{2}{m}\mathbb{E}_{\pi}\left[\left|\frac{\d \nu}{\d \pi}(\xi)-1\right|^2\right]\leq \frac{2}{m} \mathbb{E}_{\pi}\left[\left|\frac{\d \nu}{\d \pi}(\xi)\right|^2\right].
\end{equation}

The proof is completed by applying the triangle inequality to \eqref{eq:E[A/B]_bound}, then applying the bounds \eqref{eq:E[AB]_bound} and \eqref{eq:E[B^2]_bound}.
\end{proof}

\begin{proof}[Proof of proposition \ref{prop:Q_hat}]
Define:
\begin{equation}\label{eq:R_tilde_defn}\widetilde{R}(\phi;k,\tau) \coloneqq \mathbb{E}\left[ \log \left( \sum_{j=1}^k e^{\langle\phi(Z^a),\phi(Z_{j}^-) - \phi(Z^+)\rangle/\tau}\right)\right] - \log k,
\end{equation}
where $Z^a\sim \data$, $Z^+\sim M(Z^a,\cdot)$ and $Z_1^-,\ldots,Z_k^-\simiid\data$,
and decompose:
\begin{align}
& \int_{\Z} \mathbb{E}\left(\mathrm{CrossEnt}\left[\widehat{M}_k(z,\cdot)\| \widehat{Q}_{\tau,k}^\phi(z,\cdot)\right]\right) \data(\d z) - R(\phi;k,\tau)\nonumber \\ & = \int_{\Z} \mathbb{E}\left(\mathrm{CrossEnt}\left[\widehat{M}_k(z,\cdot)\| \widehat{Q}_{\tau,k}^\phi(z,\cdot)\right]\right) \data(\d z) -\widetilde{R}(\phi;k,\tau) \label{eq:R_decomp1}\\
&+ \widetilde{R}(\phi;k,\tau) - R(\phi;k,\tau). \label{eq:R_decomp2} 
\end{align}

In order to bound the difference in  \eqref{eq:R_decomp1}, let us first write out an expression for $\widetilde{R}(\phi;k,\tau)$, as follows:
\begin{align}
& \widetilde{R}(\phi;k,\tau)\nonumber\\
& = -\frac{1}{\tau}\mathbb{E}[\langle\phi(Z^a),\phi(Z^+)\rangle] + \mathbb{E}\left[ \log \left(\frac{1}{k} \sum_{j=1}^k e^{\langle\phi(Z^a),\phi(Z_{j}^-)\rangle/\tau}\right)\right]\nonumber\\
& = -\frac{1}{\tau} \int_{\Z}\int_{\Z} \langle\phi(z),\phi(z^+)\rangle M(z,\d z^+) \data(\d z)+ \mathbb{E}\left[ \log \left(\frac{1}{k} \sum_{j=1}^k e^{\langle\phi(Z^a),\phi(Z_{j}^-)\rangle/\tau}\right)\right],\label{eq:R_tilde_decomp}
\end{align}
where the first equality is just rearrangement of \eqref{eq:R_tilde_defn}, the second equality holds since $Z^a\sim\data$ and $Z^+\sim M(Z^a,\cdot)$.

Recalling the definitions: $$
\widehat{Q}_{\tau,k}^\phi(z,\d z^\prime)\coloneqq \frac{\sum_{j=1}^k e^{\langle\phi(z),\phi(\xi_j)\rangle/\tau}\delta_{\xi_j}(\d z^\prime)}{\sum_{j=1}^k e^{\langle\phi(z),\phi(\xi_j)\rangle/\tau}},\qquad \widehat{M}_k(z,\d z^\prime)\coloneqq \frac{\sum_{j=1}^k \frac{\d M(z,\cdot)}{\d \data} (\xi_j) \delta_{\xi_j}(\d z^\prime)}{\sum_{j=1}^k \frac{\d M(z,\cdot)}{\d \data} (\xi_j)}.
$$
and
$\widehat{\pi}_{\mathrm{data},k}\coloneqq k^{-1} \sum_{j=1}^k \delta_{\xi_j}$, where $\xi_1,\ldots,\xi_k\simiid \data$,
we have:
\begin{equation}\label{eq:dQ_hat}
\frac{\d \widehat{Q}_{\tau,k}^\phi(z,\cdot)}{\d \widehat{\pi}_{\mathrm{data},k}}(z^\prime) = \frac{e^{\langle\phi(z),\phi(z^\prime)\rangle/\tau}}{\frac{1}{k}\sum_{j=1}^k e^{\langle\phi(z),\phi(\xi_j)\rangle/\tau}}\quad\text{for}\quad z^\prime \in\{\xi_1,\ldots,\xi_k\},
\end{equation}
noting $\{\xi_1,\ldots,\xi_k\}$ is a subset of, or equal to, the support of $\widehat{M}_k(z,\cdot)$.
Therefore
\begin{align*}
&\int_{\Z} \mathbb{E}\left(\mathrm{CrossEnt}\left[\widehat{M}_k(z,\cdot)\| \widehat{Q}_{\tau,k}^\phi(z,\cdot)\right]\right) \data(\d z) \\
&= \int_{\Z} \mathbb{E}\left[-\int_{\Z} \log \frac{\d \widehat{Q}_{\tau,k}^\phi(z,\cdot)}{\d \widehat{\pi}_{\mathrm{data},k}}(z^\prime)\widehat{M}_k(z,\d z^\prime)\right] \data(\d z) \\
& = -\frac{1}{\tau}\int_{\Z} \mathbb{E}\left[\int_{\Z}  \langle\phi(z),\phi(z^\prime)\rangle\widehat{M}_k(z,\d z^\prime)\right] \data(\d z) +  \mathbb{E}\left[\log\left(\frac{1}{k}\sum_{j=1}^k e^{\langle\phi(Z^a),\phi(\xi_j)\rangle/\tau}\right)\right] \\
& = -\frac{1}{\tau}\int_{\Z}\mathbb{E}\left[\frac{\sum_{j=1}^k\langle\phi(z),\phi(\xi_j)\rangle \frac{\d M(z,\cdot)}{\d \data}(\xi_j)}{\sum_{j=1}^k \frac{\d M(z,\cdot)}{\d \data}(\xi_j)}\right] \data(\d z) +  \mathbb{E}\left[\log\left(\frac{1}{k}\sum_{j=1}^k e^{\langle\phi(Z^a),\phi(Z_j^-)\rangle/\tau}\right)\right], 
\end{align*}
where the first equality holds by definition of $\mathrm{CrossEnt}\left[\widehat{M}_k(z,\cdot)\| \widehat{Q}_{\tau,k}^\phi(z,\cdot)\right]$; the second equality is obtained by substituting in \eqref{eq:dQ_hat}; the third equality holds by substituting in the definition of $\widehat{M}_k$ and using the fact that $\xi_1,\ldots,\xi_k$ and $Z_1^-,\ldots,Z_k^-$ are identically distributed. Comparing with \eqref{eq:R_tilde_decomp}, we find:
\begin{multline*}
\left|\int_{\Z} \mathbb{E}\left(\mathrm{CrossEnt}\left[\widehat{M}_k(z,\cdot)\| \widehat{Q}_{\tau,k}^\phi(z,\cdot)\right]\right) \data(\d z) -\widetilde{R}(\phi;k,\tau)\right|\\
\leq \frac{1}{\tau}\int _{Z}\left|\mathbb{E}\left[\frac{\sum_{j=1}^k\langle\phi(z),\phi(\xi_j)\rangle \frac{\d M(z,\cdot)}{\d \data}(\xi_j)}{\sum_{j=1}^k \frac{\d M(z,\cdot)}{\d \data}(\xi_j)}\right]  - \int_{\Z} \langle\phi(z),\phi(z^+)\rangle M(z,\d z^+) \right| \data(\d z).
\end{multline*}
In order to control the integrand for any fixed $z$, we apply lemma \ref{lem:self_norm_is} with $\nu\coloneqq M(z,\cdot)$, $\pi\coloneqq \data$ and $f(\xi)\coloneqq\langle\phi(z),\phi(\xi)\rangle$, which satisfies $|f(\xi)|\leq 1$ as required since $\phi\in\BS$, yielding:
$$
\left|\mathbb{E}\left[\frac{\sum_{j=1}^k\langle\phi(z),\phi(\xi_j)\rangle \frac{\d M(z,\cdot)}{\d \data}(\xi_j)}{\sum_{j=1}^k \frac{\d M(z,\cdot)}{\d \data}(\xi_j)}\right]  - \int_{\Z} \langle\phi(z),\phi(z^+)\rangle M(z,\d z^+) \right| \leq \frac{4}{k}\int_{\Z} \left|\frac{\d M(z,\cdot)}{\d \data}(z^\prime) \right|^2\data(\d z^\prime),
$$
and in turn:
\begin{multline*}
\left|\int_{\Z} \mathbb{E}\left(\mathrm{CrossEnt}\left[\widehat{M}_k(z,\cdot)\| \widehat{Q}_{\tau,k}^\phi(z,\cdot)\right]\right) \data(\d z) -\widetilde{R}(\phi;k,\tau)\right|\\
\leq \frac{4}{k\tau }\int_{\Z}\int_{\Z} \left|\frac{\d M(z,\cdot)}{\d \data}(z^\prime) \right|^2\data(\d z^\prime) \data(\d z).
\end{multline*}
This completes our treatment of the difference in \eqref{eq:R_decomp1}. For the difference in  \eqref{eq:R_decomp2}, recalling the definitions of $R(\phi;k,\tau)$ and $\widetilde{R}(\phi;k,\tau)$ in \eqref{eq:R_k_defn} and \eqref{eq:R_tilde_defn}, we have:
\begin{align*}
0&\leq R(\phi;k,\tau) - \widetilde{R}(\phi;k,\tau)\\
&=\mathbb{E}\left[\log\left(1+\frac{1}{\sum_{j=1}^k e^{\langle\phi(Z^a),\phi(Z_j^-)-\phi(Z^+)\rangle/\tau}}\right)\right]\\
&\leq \log\left(1+\frac{e^{2/\tau}}{k}\right),
\end{align*}
where the first inequality uses monotonicity of $\log$, the equality uses $\log(1/k+c)-\log(c) = \log(1+1/kc)$, and the second inequality uses $\phi\in\BS$.

Having thus obtained bounds on the absolute values of the differences in \eqref{eq:R_decomp1}-\eqref{eq:R_decomp2}, the proof is completed by applying the triangle inequality there.

\end{proof}

\begin{lem}\label{lem:determ_bound_R_Tilde} For any $\phi\in\BS$, $k\geq1$ and $\tau >0$,
$$
\int_{\Z}\mathrm{CrossEnt}[M(z,\cdot)\| Q_\tau^\phi(z,\cdot)]\data(\d z)- \widetilde{R}(\phi;k,\tau)  \leq \frac{2}{\tau}.
$$
\end{lem}
\begin{proof}
\begin{align*}
& \int_{\Z}\mathrm{CrossEnt}[M(z,\cdot)\| Q_\tau^\phi(z,\cdot)]\data(\d z)- \widetilde{R}(\phi;k,\tau) \\
& = \mathbb{E}\left[\log \int_{\Z} e^{\langle\phi(Z),\phi(z^-)-\phi(Z^+)\rangle/\tau }\data(\d z^-)\right]- \mathbb{E}\left[\log \frac{1}{k}\sum_{j=1}^k e^{\langle\phi(Z),\phi(Z_j^-)-\phi(Z^+)\rangle/\tau }\right]\\
&  = \mathbb{E}\left[\log \int_{\Z} e^{\langle\phi(Z),\phi(z^-) \rangle/\tau }\data(\d z^-)\right]- \mathbb{E}\left[\log \frac{1}{k}\sum_{j=1}^k e^{\langle\phi(Z),\phi(Z_j^-) \rangle/\tau }\right]\\
& = \mathbb{E}\left[\log \frac{\int_{\Z} e^{\langle\phi(Z),\phi(z^-) \rangle/\tau }\data(\d z^-)}{\frac{1}{k}\sum_{j=1}^k e^{\langle\phi(Z),\phi(Z_j^-) \rangle/\tau }}\right] \leq \log \frac{e^{1/\tau}}{e^{-1/\tau}} = \frac{2}{\tau},
\end{align*}
where the inequality holds since $\|\phi(z)\|_2=1$ for all $z$.  
\end{proof}

\subsection{Calculations for section \ref{sec:simplified_bounds}}\label{app:calcs_for_DCL}

The bound in \eqref{eq:R_tilde_prop2} is obtained as part of the proof of proposition \ref{prop:Q_hat}, see \eqref{eq:R_decomp1} in particular. To check the bound in \eqref{eq:R_tilde_prop1}, write out the difference:
\begin{align}
&\int_{\Z}\mathrm{CrossEnt}[M(z,\cdot)\| Q_\tau^\phi(z,\cdot)]\data(\d z)- \widetilde{R}(\phi;k,\tau)\nonumber \\
& = -\frac{1}{\tau}\int_{\Z}\langle\phi(z),\phi(z^+)\rangle \data(\d z) M(z,\d z^+) + \int_{\Z} \log \left[\int_{\Z} e^{\langle\phi(z),\phi(z^-)\rangle/\tau } \data(\d z^-)\right] \data(\d z)\nonumber \\
&\quad +\frac{1}{\tau}\int_{\Z}\langle\phi(z),\phi(z^+)\rangle \data(\d z) M(z,\d z^+) - \mathbb{E}\left[\log\left( \frac{1}{k}\sum_{j=1}^k e^{\langle\phi(Z),\phi(Z_j^-)\rangle/\tau } \right) \right]\nonumber \\
& = \int_{Z} \log \left[\int_{\Z} e^{\langle\phi(z),\phi(z^-)\rangle/\tau } \data(\d z^-)\right] -  \mathbb{E}\left[\log\left( \frac{1}{k}\sum_{j=1}^k e^{\langle\phi(z),\phi(Z_j^-)\rangle/\tau } \right) \right]  \data(\d z)\nonumber\\
& = \int_{Z} \log
 \mu(z)  - \mathbb{E}[\log \widehat{\mu}(z) ] \,\data(\d z),\label{eq:cross_ent_id_ex}
\end{align}
where the final equality holds with the shorthand notation:
$$
\mu(z)\coloneqq \int_{\Z} e^{\langle\phi(z),\phi(z^-)\rangle/\tau } \data(\d z^-),\qquad \widehat{\mu}(z)\coloneqq  \frac{1}{k}\sum_{j=1}^k e^{\langle\phi(z),\phi(Z_j^-)\rangle/\tau }.
$$

By Jensen's inequality, $\log \mu(z) \geq  \mathbb{E}[\log \widehat{\mu}(z) ] $. The  upper-bound in \eqref{eq:R_tilde_prop1} is derived by making a Taylor expansion of $y\mapsto \log(y)$ similarly to as in the proof of lemma \ref{lem:bias_taylor}. The details are omitted, but in order to perform calculations for a specific example, consider the Taylor expansion up to third order (with $z$ fixed):
\begin{align*}
\log \widehat{\mu}(z) &= \log \mu(z) + \frac{\widehat{\mu}(z) - \mu(z)}{\mu(z)} - \frac{(\widehat{\mu}(z) - \mu(z))^2}{2 \mu(z)^2} + \frac{(\widehat{\mu}(z) - \mu(z))^3}{3\xi^3},
\end{align*}
for some $\xi$ on the line segment between $\mu(z)$ and $\widehat{\mu}(z)$.

Using $\mu(z)\vee \widehat{\mu}(z) \leq e^{1/\tau}$, we thus have:
$$
\log \mu(z) - \mathbb{E}[\log \widehat{\mu}(z) ] \geq   \frac{\mathbb{E}[(\widehat{\mu}(z) - \mu(z))^2]}{2 \mu(z)^2} - \frac{\mathbb{E}[|\widehat{\mu}(z) - \mu(z)|^3]}{3 e^{3/\tau}}.
$$

Since $Z_1^-,\ldots,Z_k^-$ are i.i.d., the Marcinkiewicz–Zygmund inequality gives $\mathbb{E}[|\widehat{\mu}(z) - \mu(z)|^3] = O(k^{-3/2})$ uniformly in $z$, and by direct calculation,
$$
\mathbb{E}[(\widehat{\mu}(z) - \mu(z))^2] = \frac{1}{n} \mathrm{Var}[e^{\langle\phi(z),\phi(Z_1^-)\rangle/\tau}],
$$
where $Z_1^-\sim \data$.

This shows that 
\begin{equation}\label{eq:bias_asymp_lower_bound}
k \log \mu(z) - k\mathbb{E}[\log \widehat{\mu}(z) ] \geq   \frac{1}{2}\frac{\mathrm{Var}[e^{\langle\phi(z),\phi(Z_1^-)\rangle/\tau}]}{\mathbb{E}[e^{\langle\phi(z),\phi(Z_1^-)\rangle/\tau}]^2}- O(k^{-1/2}),
\end{equation}
as $k\to\infty$.

\subsection{An example}\label{app:example}

Now let us construct an example for which we shall lower bound $\mathrm{Var}[e^{\langle\phi(z),\phi(Z_1^-)\rangle/\tau}]/\mathbb{E}[e^{\langle\phi(z),\phi(Z_1^-)\rangle/\tau}]^2$. The idea of the construction is to make $|\langle\phi(z),\phi(Z_1^-)\rangle|$ close to zero with probability close to $1$, and otherwise $\langle\phi(z),\phi(Z_1^-)\rangle=1$.

Assume $e^{1/\tau} >2$ and define 
\begin{equation}\label{eq:eps_defn}
\epsilon \coloneqq \frac{1}{\lceil e^{1/\tau}-1\rceil},
\end{equation}
so that $\epsilon \in(0,1)$. Let $\delta\in(0,1)$ and let $d(\delta)$ be large enough that there exist $1/\epsilon$ vectors in $\Sd$ which are $\delta$-orthogonal, i.e., vectors $u_1,\ldots,u_{1/\epsilon}$ such that for $i\neq j$, $|\langle u_i,u_j\rangle|\leq \delta$. Over the course of the following construction we shall consider taking $\delta\to 0$ and when we do it will be automatically assumed that $d$ grows suitably fast as $\delta$ shrinks that $1/\epsilon$ $\delta$-orthogonal vectors exist ($d=1/\epsilon$ is sufficient, since in that case the existence of $1/\epsilon$ orthogonal vectors in $\Sd$ is trivial, but we allow for $\delta$-orthogonality rather than strict orthogonality to emphasise that $d=1/\epsilon$ is not necessary for the construction).

Suppose that $\data$ and $\phi$ are such that for $Z\sim\data$, the random vector $\phi(Z)$ is uniformly distributed on the set $\{u_1,\ldots,u_{1/\epsilon}\}$. It follows that for any $z\in \Z$ and $Z_j^- \sim \data$, $\langle\phi(z),\phi(Z_j^-)\rangle=1$ with probability $\epsilon$, and with probability $1-\epsilon$, $\langle\phi(z),\phi(Z_j^-)\rangle=X_0$ for some random variable $X_0$ such that $|X_0|\leq \delta$.

We have:
$$
\frac{\mathrm{Var}[e^{\langle\phi(z),\phi(Z_1^-)\rangle/\tau}]}{\mathbb{E}[e^{\langle\phi(z),\phi(Z_1^-)\rangle/\tau}]^2} = \frac{\epsilon(1-\epsilon)\left(e^{1/\tau} - \mathbb{E}[e^{X_0/\tau}]\right)^2 + (1-\epsilon)\mathrm{Var}[e^{X_0/\tau}]}{\left(\epsilon e^{1/\tau} + (1-\epsilon)\mathbb{E}[e^{X_0/\tau}]\right)^2}.
$$

As $\delta\to 0$ (with $d$ increasing as necessary for the required number of $\delta$-orthogonal vectors to exist), the r.h.s. of the above tends to:
$$
\frac{\epsilon(1-\epsilon)\left(e^{1/\tau} - 1\right)^2  }{\left(1 + \epsilon (e^{1/\tau} -1)\right)^2}.
$$
Therefore for any $a\in(0,1)$ we can choose $\delta>0$ small enough that:
$$
\frac{\mathrm{Var}[e^{\langle\phi(z),\phi(Z_1^-)\rangle/\tau}]}{\mathbb{E}[e^{\langle\phi(z),\phi(Z_1^-)\rangle/\tau}]^2} \geq (1-a) \frac{\epsilon(1-\epsilon)\left(e^{1/\tau} - 1\right)^2  }{\left(1 + \epsilon (e^{1/\tau} -1)\right)^2}.
$$
Using $e^{1/\tau}>2$ and \eqref{eq:eps_defn}  we have:  $$\frac{1}{2}\leq \frac{e^{1/\tau}-1}{e^{1/\tau}}\leq \epsilon(e^{1/\tau}-1)\leq 1,$$ and so:
\begin{align*}
\frac{\mathrm{Var}[e^{\langle\phi(z),\phi(Z_1^-)\rangle/\tau}]}{\mathbb{E}[e^{\langle\phi(z),\phi(Z_1^-)\rangle/\tau}]^2} &\geq (1-a) \frac{(1-\epsilon)\left(e^{1/\tau} - 1\right)^2  }{4 e^{1/\tau}}\\
& \geq \frac{(1-a)}{8 } \left( e^{1/\tau}-1-\frac{e^{1/\tau}-1 }{\lceil e^{1/\tau}-1\rceil}\right)\\
&\geq \frac{(1-a)}{8}(e^{1/\tau}-2).
\end{align*}
Combining with \eqref{eq:cross_ent_id_ex} and  \eqref{eq:bias_asymp_lower_bound} we find that for this example:
$$
\liminf _{k\to\infty}\left(\int_{\Z}\mathrm{CrossEnt}[M(z,\cdot)\| Q_\tau^\phi(z,\cdot)]\data(\d z)- \widetilde{R}(\phi;k,\tau)\right) \geq \frac{(1-a)}{16}(e^{1/\tau}-2),
$$
from which \eqref{eq:ratio_bound} follows by combining with \eqref{eq:R_tilde_prop2}.

\section{Proofs and supporting results for section \ref{sec:generalisation}}

\begin{lem}\label{lem:ell_determ_bound} For any $k\geq1$, $\tau>0$, $\mathbf{z}\in\Z^{2+k}$, and $\phi,\phi^\prime\in\BS$,
$$
\left|\ell(\phi,\mathbf{z},k,\tau)-\ell(\phi^\prime,\mathbf{z},k,\tau)\right| \leq \frac{4}{\tau}.
$$
\end{lem}
\begin{proof}
Since $\|\phi\|_{2,\infty}=1$, we have:
$$
\log\left(\frac{1}{k}+e^{-2/\tau}\right)\leq\ell(\phi,\mathbf{z},k,\tau)\leq \log\left(\frac{1}{k}+e^{2/\tau}\right),
$$
with the same inequality holding with $\phi$ replaced by $\phi^\prime$. The claim of the lemma then follows by bounding:
\begin{align*}
 |\ell(\phi,\mathbf{z},k,\tau)- \ell(\phi^\prime ,\mathbf{z},k,\tau)|&\leq  \log\left(\frac{1+k e^{2/\tau}}{1+k e^{-2/\tau}}\right)\\
&=\log\left(e^{2/\tau}\left[\frac{1+k e^{2/\tau}}{e^{2/\tau}+k }\right]\right)\\
&\leq \log\left(e^{2/\tau}\left[\frac{e^{2/\tau}+k e^{2/\tau}}{1+k }\right]\right)
=\frac{4}{\tau},
\end{align*}
where the inequality uses $e^{2/\tau}\geq 1$.

\end{proof}

\begin{proof}[Proof of proposition \ref{prop:uniform}.] Noting that $R(\phi;k,\tau)=\mathbb{E}\left[\widehat{R}_n(\phi;k,\tau)\right]$, proposition \ref{prop:uniform} is almost a direct  application of \citep[Thm. 3.3]{mohri2018foundations}. The latter theorem applies to sample averages of functions with co-domain $[0,1]$, whereas in the present setting we have lower and upper bounds:
$$
a\coloneqq\log\left(\frac{1}{k}+e^{-2/\tau}\right)\leq\ell(\phi,\mathbf{z}_i,k,\tau)\leq \log\left(\frac{1}{k}+e^{2/\tau}\right)=:b,
$$
which hold for any $\phi\in\BS$ and any $\mathbf{z}_i\in\Z^{2+k}$.

To obtain a function with co-domain $[0,1]$ denote:
$$
g(\phi,\mathbf{z}_i)\coloneqq \frac{\ell(\phi,\mathbf{z}_i,k,\tau)-a}{b-a},
$$
where the dependence on $\tau,k$ is hidden for ease of presentation. An application of \citep[Thm. 3.3]{mohri2018foundations} to  $n^{-1}\sum_{i=1}^n g(\phi,\mathbf{Z}_i) $ gives that with probability at least $1-\delta$, for all $\phi\in\Phi$,
\begin{equation}\label{eq:Mohri_ulln}
\mathbb{E}[g(\phi,\mathbf{Z}_1)] \leq \frac{1}{n}\sum_{i=1}^n g(\phi,\mathbf{Z}_i) + 2 \mathbb{E}\left[\left.\sup_{\phi\in\Phi}\frac{1}{n}\sum_{i=1}^n\sigma_i g(\phi,\mathbf{Z}_i)\right|\mathbf{Z}_1,\ldots,\mathbf{Z}_n\right] + 3\sqrt{\frac{\log \frac{2}{\delta}}{2n}}. 
\end{equation}
Since Rademacher complexity of a class of functions is invariant to adding a constant scalar to every member of the class, and factorising out $b-a$, we have:
$$
\mathbb{E}\left[\left.\sup_{\phi\in\Phi}\frac{1}{n}\sum_{i=1}^n\sigma_i g(\phi,\mathbf{Z}_i)\right|\mathbf{Z}_1,\ldots,\mathbf{Z}_n\right] = \frac{1}{b-a} \mathbb{E}\left[\left.\sup_{\phi\in\Phi}\frac{1}{n}\sum_{i=1}^n\sigma_i \ell(\phi,\mathbf{Z}_i,k,\tau)\right|\mathbf{Z}_1,\ldots,\mathbf{Z}_n\right].
$$
Substituting into \eqref{eq:Mohri_ulln}, multiplying both sides of the inequality by $b-a$,  adding $a$ to both sides and recalling the definitions \eqref{eq:R_k_defn} and \eqref{eq:R_hat_defn} gives:
$$
R(\phi;k,\tau)\leq \widehat{R}_n(\phi;k,\tau)+ 2 \mathbb{E}\left[\left.\sup_{\phi\in\Phi}\frac{1}{n}\sum_{i=1}^n\sigma_i \ell(\phi,\mathbf{Z}_i,k,\tau)\right|\mathbf{Z}_1,\ldots,\mathbf{Z}_n\right] + 3(b-a)\sqrt{\frac{\log \frac{2}{\delta}}{2n}}.
$$
The proof is completed by upper-bounding $b-a$ as in the proof of lemma \ref{lem:ell_determ_bound}.

\end{proof}

\subsection{Supporting results and proof of proposition \ref{prop:ell_lipschitz}}

A G\^{a}teaux derivative can be thought of as a generalisation of the directional derivative in Euclidean space.  We shall consider G\^{a}teaux derivatives of functionals mapping $\BR$ into $\mathbb{R}$. For such a functional, say $H:\BR\to\mathbb{R}$, the G\^{a}teaux derivative of $H$ at $f\in \BR$ in direction $\eta\in\BR$ is the limit (if it exists):
$$
\delta H(f)[\eta]\coloneqq \lim_{\epsilon \to 0} \frac{H(f+\epsilon \eta)-H(f)}{\epsilon}.
$$

To prepare for the proof of proposition \ref{prop:ell_lipschitz}, we need the following definitions. For any $\tau>0$, $x,y\in\mathcal{Z}$ and $\mu\in\P(\Z)$,
define the functionals $H_{\tau}(\cdot,x,y,\mu):\BR\to\mathbb{R}$
and $G_{\tau}(\cdot,\cdot,x,y,\mu):\BR\times \BR\to\mathbb{R}^{d}$,
\begin{align}
H_{\tau}(f,x,y,\mu) & \coloneqq\log\int_{\mathcal{Z}}e^{\left\langle f(x),f(z)-f(y)\right\rangle /\tau}\mu(\d z).\label{eq:H_defn}\\
G_{\tau}(f,\eta,\mu,x) & \coloneqq\dfrac{\int_{\mathcal{Z}}\eta(z)e^{\left\langle f(x),f(z)\right\rangle /\tau}\mu(\d z)}{\int_{\mathcal{Z}}e^{\left\langle f(x),f(z)\right\rangle /\tau}\mu(\d z)},\label{eq:G_defn}
\end{align}
where in the numerator of \eqref{eq:G_defn} the vector-valued function $\eta$ is integrated elementwise.

\begin{lem}\label{lem:deltaH}
For any $\tau>0$, $x,y\in\mathcal{Z}$ and $\mu\in\mathcal{P}(\mathcal{Z})$,
the G\^{a}teaux derivative of the functional $H_{\tau}(\cdot,x,y,\mu)$
at a point $f\in \BR$, in direction $\eta\in \BR$,
is:
\[
\delta H_{\tau}(f,x,y,\mu)[\eta]\coloneqq\frac{1}{\tau}\left[\left\langle \eta(x),G_{\tau}(f,f,x,\mu)-f(y)\right\rangle +\left\langle f(x),G_{\tau}(f,\eta,\mu,x)-\eta(y)\right\rangle \right].
\]
\end{lem}
\begin{proof}
For brevity throughout the proof we write ``derivative'' instead
of ``G\^{a}teaux derivative''. From (\ref{eq:H_defn}) we have:
\begin{equation}
H_{\tau}(f,x,y,\mu)=-\frac{1}{\tau}\left\langle f(x),f(y)\right\rangle +\log\int_{\mathcal{Z}}e^{\left\langle f(x),f(z)\right\rangle /\tau}\mu(\mathrm{d}z).\label{eq:H_decomp}
\end{equation}
Considering the first term on the l.h.s. of (\ref{eq:H_decomp}), for any $\epsilon>0$,
\[
\frac{1}{\tau}\left\langle f(x)+\epsilon\eta(x),f(y)+\epsilon\eta(y)\right\rangle -\frac{1}{\tau}\left\langle f(x),f(y)\right\rangle =\frac{\epsilon}{\tau}\left[\left\langle f(x),\eta(y)\right\rangle +\left\langle \eta(x),f(y)\right\rangle \right]+\frac{\epsilon^{2}}{\tau}\left\langle \eta(x),\eta(y)\right\rangle ,
\]
then using Cauchy-Schwartz and $\|\eta\|_{2,\infty}<\infty$, the
derivative of $f\mapsto\frac{1}{\tau}\left\langle f(x),f(y)\right\rangle $
in direction $\eta$ is:
\begin{equation}
\lim_{\epsilon\to0}\frac{\frac{1}{\tau}\left\langle f(x)+\epsilon\eta(x),f(y)+\epsilon\eta(y)\right\rangle -\frac{1}{\tau}\left\langle f(x),f(y)\right\rangle }{\epsilon}=\frac{1}{\tau}\left[\left\langle f(x),\eta(y)\right\rangle +\left\langle \eta(x),f(y)\right\rangle \right].\label{eq:H_deriv_term_1}
\end{equation}

To find the derivative of the second term on the r.h.s. of (\ref{eq:H_decomp})
we use the chain rule. So first consider the functional $f\mapsto \Gamma(f)\coloneqq e^{\left\langle f(x),f(z)\right\rangle /\tau}$ (where dependence on $\tau$, $x$ and $z$ is suppressed from the notation).
Using (\ref{eq:H_deriv_term_1}) with $y$ replaced by $z$, the derivative
of $\Gamma$ at $f$
in direction $\eta$ is:
\begin{equation}\label{eq:deriv_of_e_to_prod}
\delta\Gamma(f)[\eta]\coloneqq\frac{1}{\tau}\left[\left\langle f(x),\eta(z)\right\rangle +\left\langle \eta(x),f(z)\right\rangle \right]e^{\left\langle f(x),f(z)\right\rangle /\tau}.
\end{equation}

Our next objective is to show that the derivative of $f\mapsto\int_{\mathcal{Z}}e^{\left\langle f(x),f(z)\right\rangle /\tau}\mu(\mathrm{d}z)$ is given by \eqref{eq:deriv_of_e_to_prod} with $z$ integrated out under $\mu$. For this purpose we seek to apply the dominated convergence theorem, as follows.  Fix any $\epsilon\in(0,1]$,  $f\in\BR$ and $\eta\in \BR$.
By the mean value theorem there exists $c\in[0,1]$   such that by evaluating $\Gamma(\cdot)[\epsilon \eta]$ at the point:
$(1-c)f+c\left(f+\epsilon\eta\right)=f+\epsilon c\eta$,
\begin{align*}
&\frac{1}{\epsilon}\left|e^{\left\langle f(x)+\epsilon\eta(x),f(z)+\epsilon\eta(z)\right\rangle /\tau}-e^{\left\langle f(x),f(z)\right\rangle /\tau}\right|\\ 
& =\frac{1}{\epsilon}\left|\delta\Gamma(f+\epsilon c\eta)[\epsilon \eta]\right|\\
& =\frac{1}{\tau\epsilon}\left|\left\langle f(x)+\epsilon c\eta(x),\epsilon\eta(z)\right\rangle +\left\langle \epsilon\eta(x),f(z)+\epsilon c\eta(z)\right\rangle \right|e^{\left\langle f(x)+\epsilon c\eta(x),f(z)+\epsilon c\eta(z)\right\rangle /\tau}\\
 & =\frac{1}{\tau}\left|\left\langle f(x)+\epsilon c\eta(x),\eta(z)\right\rangle +\left\langle \eta(x),f(z)+\epsilon c\eta(z)\right\rangle \right|e^{\left\langle f(x)+\epsilon c\eta(x),f(z)+\epsilon c\eta(z)\right\rangle /\tau}\\
 & \leq\frac{1}{\tau}\left[\|f+\epsilon c\eta\|_{2,\infty}\|\eta\|_{2,\infty}+\|\eta\|_{2,\infty}\|f+\epsilon c\eta\|_{2,\infty}\right]e^{(\|f\|_{2,\infty}+\|\eta\|_{2,\infty})^2/\tau}\\
 & \leq\frac{2}{\tau}\left(\|f\|_{2,\infty}+\|\eta\|_{2,\infty}\right)\|\eta\|_{2,\infty}e^{(\|f\|_{2,\infty}+\|\eta\|_{2,\infty})^2/\tau}<\infty.
\end{align*}
Since $\epsilon$ was any value in $(0,1]$, the dominated convergence theorem allows interchange of integration and differentiation such that from \eqref{eq:deriv_of_e_to_prod} the derivative of
$f\mapsto\int_{\mathcal{Z}}e^{\left\langle f(x),f(z)\right\rangle /\tau}\mu(\mathrm{d}z)$
at $f$ in direction $\eta$ is:
\[
\int_{\mathcal{Z}}\frac{1}{\tau}\left[\left\langle f(x),\eta(z)\right\rangle +\left\langle \eta(x),f(z)\right\rangle \right]e^{\left\langle f(x),f(z)\right\rangle /\tau}\mu(\mathrm{d}z).
\]
By one further application of the chain rule, the derivative of $f\mapsto\log\int_{\mathcal{Z}}e^{\left\langle f(x),f(z)\right\rangle /\tau}\mu(\mathrm{d}z)$
in direction $\eta$ is:
\begin{multline}
\dfrac{\int_{\mathcal{Z}}\frac{1}{\tau}\left[\left\langle f(x),\eta(z)\right\rangle +\left\langle \eta(x),f(z)\right\rangle \right]e^{\left\langle f(x),f(z)\right\rangle /\tau}\mu(\mathrm{d}z)}{\int_{\mathcal{Z}}e^{\left\langle f(x),f(z)\right\rangle /\tau}\mu(\mathrm{d}z)}\\
=\frac{1}{\tau}\left[\left\langle f(x),G_{\tau}(f,\eta,\mu,x)\right\rangle +\left\langle \eta(x),G_{\tau}(f,f,\mu,x)\right\rangle \right],\label{eq:H_deriv_term_2}
\end{multline}
where the definition \eqref{eq:G_defn} has been used. Recalling (\ref{eq:H_decomp}), the proof is completed by subtracting
(\ref{eq:H_deriv_term_1}) from (\ref{eq:H_deriv_term_2}).
\end{proof}

\begin{lem}\label{lem:fractional_power}
For any $k \geq 1$, $\tau>0$, $\beta \geq 1$, $x_1,\ldots,x_k\in[-1,1]$ and $a_1,\ldots,a_k \geq 0$,
$$
\frac{\sum_{j=1}^k e^{x_j/\tau} a_j}{\sum_{j=1}^k e^{x_j/\tau}} \leq e^{2/(\beta\tau)}\left(\frac{1}{k}\sum_{j=1}^k  a_j^{\beta}\right)^{1/\beta}. 
$$
\end{lem}
\begin{proof}
Denote $p_j \coloneqq e^{x_j/\tau} /\sum_{i=1}^k e^{x_i/\tau}$.
Since $\beta \geq 1$, Jensen's inequality gives:
$$
\left(\sum_{j=1}^k p_j a_j\right)^{\beta} \leq \sum_{j=1}^k p_j a_j^{\beta}.
$$
Combining this inequality with the fact that $p_j\leq e^{2/\tau}/k$ gives: 
$$
\sum_{j=1}^k p_j a_j \leq \left(\sum_{j=1}^k p_j a_j^{\beta}\right)^{1/\beta}\leq e^{2/(\beta\tau)}\left(\frac{1}{k}\sum_{j=1}^k  a_j^{\beta}\right)^{1/\beta}.
$$

\end{proof}

\begin{lem}\label{lem:G_phi_phi} For any $k\geq 1$, $\tau>0$, $(z^a,z_1^-,\ldots,z_k^-)\in\Z^{1+k}$, $\phi,\phi^\prime\in\BS$, $f\in\BR$ such that $\|f\|_{2,\infty}\leq 1$,  and $\beta\geq 1$,
$$
\|G_\tau(f,\phi-\phi^\prime,\widehat{\pi},z^a)\|_2 \leq e^{2/{\beta\tau}}\left(\frac{1}{k}\sum_{j=1}^k \|\phi(z_j^-)-\phi^\prime(z_j^-)\|_2^{\beta}\right)^{1/\beta},
$$
where $\widehat{\pi}\coloneqq k^{-1}\sum_{j=1}^k \delta_{z_j^-}$.
\end{lem}
\begin{proof}
As in the statement, fix any $k\geq 1$, $\tau>0$, $(z^a,z_1^-,\ldots,z_k^-)\in\Z^{1+k}$, $\phi,\phi^\prime\in\BS$ and $f\in\BR$ such that $\|f\|_{2,\infty}\leq 1$. Define the shorthand:
$$
p_{j}\coloneqq \frac{e^{\langle f(z^a),f(z_{j}^-) \rangle/\tau}}{\sum_{l=1}^k e^{\langle f(z^a),f(z_{\ell}^-) \rangle/\tau}},\qquad j=1,\ldots,k.
$$
From the definition of $G_{\tau}$ in \eqref{eq:G_defn},
$$
G_\tau(f,\phi-\phi^\prime,\widehat{\pi},z^a) = \sum_{j=1}^k p_j \left[\phi(z_j^-) - \phi^\prime(z_j^-)\right].
$$
Now choose any $\beta\geq 1$. By an application of the triangle inequality for the $\|\cdot\|_2$ norm and lemma \ref{lem:fractional_power} with there $x_j=\langle f(z^a),f(z_j^-)\rangle$ (so that $x_j\in[-1,1]$ as required since by assumption of the present lemma $\|f\|_{2,\infty}\leq 1$), 
\begin{equation}\label{eq:metric_bound_tem_1a}
\|G_\tau(f,\phi-\phi^\prime,\widehat{\pi},z^a)\|_2\leq \sum_{j=1}^k p_{j}  \left\| \phi(z_{j}^-)- \phi^{\prime}(z_{j}^-)\right\|_2 \leq e^{2/(\beta\tau)} \left(\frac{1}{k}\sum_{j=1}^k \left\| \phi(z_{j}^-)- \phi^{\prime}(z_{j}^-)\right\|_2^{\beta} \right)^{1/\beta}.
\end{equation}

\end{proof}

\begin{proof}[Proof of proposition \ref{prop:ell_lipschitz}]
Fix any $k$, $\tau$ and $\mathbf{z}=(z^a,z^+,z_1^-,\ldots,z_k^-)\in\Z^{2+k}$ as in the statement of the theorem. To lighten notation in the proof, for any $f\in \BR$ define the shorthand $\ell(f)\equiv\ell(f,\mathbf{z},k,\tau)$ and $H(f)\equiv H_{\tau}(f,z^a,z^+,\widehat{\pi}^-)$, with $\widehat{\pi}^-\coloneqq k^{-1}\sum_{j=1}^k \delta_{z_{j}^+}$ and where $H_{\tau}$ is defined in \eqref{eq:H_defn}. Observe then:
$$
\ell(f) = \log\left(\frac{1}{k}+e^{H(f)}\right).
$$
Denoting by  $\delta \ell(f)[\eta]$,    $\delta e^{H(f)}[\eta]$ and $\delta H(f)[\eta]$ the G\^{a}teaux derivatives of respectively $f\mapsto\ell(f)$, $f\mapsto e^{H(f)}$ and $f\mapsto H(f)$ at a point $f\in\BR$ and in direction $\eta\in\BR$, the chain rule gives:
\begin{equation}\label{eq:delta_ell}
\delta \ell(f)[\eta] = \frac{0+\delta e^{H(f)}[\eta]}{1/k +e^{H(f)}}= \frac{e^{H(f)}}{1/k +e^{H(f)}}\delta H(f)[\eta]. 
\end{equation}
Now fix any $\phi,\phi^\prime\in \BS$. By the mean value theorem and \eqref{eq:delta_ell}, there exists $c\in[0,1]$, such that, with
$\xi\coloneqq (1-c)\phi+c\phi^{\prime}\in\BR$, 
\begin{align}
  \ell(\phi)-\ell(\phi^\prime) =\delta \ell(\xi)[\phi-\phi^{\prime}] = \frac{e^{H(\xi)}}{1/k +e^{H(\xi)}} \delta H(\xi)[\phi-\phi^{\prime}],\label{eq:ell-ell}
 \end{align}
 and by applying lemma \ref{lem:deltaH} with there $f=\xi$, $\eta=\phi-\phi^\prime$, $x=z^a$, $y=z^+$, $\mu = \widehat{\pi}^-$,
 \begin{align}
  &\tau\, \delta H(\xi)[\phi-\phi^{\prime}]\nonumber\\
 & =  \left[\left\langle \phi(z^a)-\phi^\prime(z^a),G_{\tau}(\xi,\xi,z^a,\widehat{\pi})-\xi(z^+)\right\rangle +\left\langle \xi(z^a),G_{\tau}(\xi,\phi-\phi^\prime,\widehat{\pi},z^a)-\phi(z^+) + \phi^\prime(z^+)\right\rangle \right]\nonumber \\
 & = \left\langle \phi(z^a)-\phi^\prime(z^a),-\xi(z^+)\right\rangle+
 \left\langle \phi(z^+) - \phi^\prime(z^+),-\xi(z^a)\right\rangle\nonumber \\
 &\quad + \left\langle \phi(z^a) - \phi^\prime(z^a),G_{\tau}(\xi,\xi,\widehat{\pi},z^a)\right\rangle
 + \left\langle \xi(z^a),G_{\tau}(\xi,\phi-\phi^\prime,\widehat{\pi},z^a)\right\rangle
 \nonumber \\
 &= \left\langle \phi(z^a)-\phi^\prime(z^a),-\xi(z^+)\right\rangle+
 \left\langle \phi(z^+) - \phi^\prime(z^+),-\xi(z^a)\right\rangle\label{eq:delta_H_1}\\
 &\quad+\left\langle \phi(z^a) - \phi^\prime(z^a),G_{\tau}(\xi,\xi,\widehat{\pi},z^a) - G_{\tau}(\xi,\phi-\phi^\prime,\widehat{\pi},z^a)\right\rangle\label{eq:delta_H_2}\\
&\quad+\left\langle\xi(z^a)+ \phi(z^a) - \phi^\prime(z^a), G_{\tau}(\xi,\phi-\phi^\prime,\widehat{\pi},z^a)\right\rangle.\label{eq:delta_H_3}
\end{align}

We shall apply the Cauchy-Schwartz inequality to each of the inner-products in \eqref{eq:delta_H_1}-\eqref{eq:delta_H_3}. In preparation, observe that since $\|\phi\|_{2,\infty} = \|\phi^\prime\|_{2,\infty}=1$, we have:
\begin{equation}\label{eq:xi_diff}
\|\xi\|_{2,\infty}\leq 1\quad\mathrm{and}\quad\|\xi-\phi+\phi^\prime\|_{2,\infty}\vee\|\xi+\phi-\phi^\prime\|_{2,\infty}\leq 3.
\end{equation}

Recalling the definition of $G_{\tau}$ in \eqref{eq:G_defn}, note that $G_{\tau}$ is linear in its second argument. Combined with Jensen's inequality, this gives:
\begin{equation}\label{eq:G_diff_bound}
\|G_{\tau}(\xi,\xi,\widehat{\pi},z^a) - G_{\tau}(\xi,\phi-\phi^\prime,\widehat{\pi},z^a)\|_2 = \|G_{\tau}(\xi,\xi-\phi+\phi^\prime,\widehat{\pi},z^a) \|_2 \leq \|\xi-\phi+\phi^\prime\|_{2,\infty}\leq 3. 
\end{equation}

Since $\|\xi\|_{2,\infty}\leq 1$ we may apply lemma \ref{lem:G_phi_phi} with there $f=\xi$ to give, for any $\beta\geq 1$,
\begin{equation}\label{eq:G_phi_phi_apply}
\|G_{\tau}(\xi,\phi-\phi^\prime,\widehat{\pi},z^a)\|_2 \leq  e^{2/{\beta\tau}}\left(\frac{1}{k}\sum_{j=1}^k \|\phi(z_j^-)-\phi^\prime(z_j^-)\|_2^{\beta}\right)^{1/\beta}.
\end{equation}

Combining \eqref{eq:ell-ell}; the fact $e^{H(\xi)}/(1/k + e^{H(\xi)})\leq 1$; application of the Cauchy-Schwartz inequality to each of the inner-products in \eqref{eq:delta_H_1}-\eqref{eq:delta_H_3}; and the bounds \eqref{eq:xi_diff}, \eqref{eq:G_diff_bound} and \eqref{eq:G_phi_phi_apply} gives: 
\begin{multline*}
|\ell(\phi)-\ell(\phi^\prime)|
 \\ \leq \frac{4}{\tau}\|\phi(z^a)-\phi^\prime(z^a)\|_2\nonumber  +\frac{1}{\tau}\|\phi(z^+)-\phi^\prime(z^+)\|_2 \nonumber  +\frac{3}{\tau}  e^{2/{\beta\tau}}\left(\frac{1}{k}\sum_{j=1}^k \|\phi(z_j^-)-\phi^\prime(z_j^-)\|_2^{\beta}\right)^{1/\beta}.
 \end{multline*}
\end{proof}

\subsection{Other proofs for section \ref{sec:generalisation}}

\begin{proof}[Proof of lemma \ref{lem:metric}]
It follows from proposition \ref{prop:ell_lipschitz} that:
\begin{multline*}
\left|\ell(\phi,\mathbf{z}_i,k,\tau) - \ell(\phi^\prime,\mathbf{z}_i,k,\tau)  \right| \\
\leq \frac{4}{\tau} \left(\|\phi(z_i^a) - \phi^\prime(z_i^a)\|_2+\|\phi(z_i^+) - \phi^\prime(z_i^+)\|_2+e^{2/\beta\tau}\left(\frac{1}{k}\sum_{j=1}^k  \|\phi(z_{ij}^-) - \phi^\prime (z_{ij}^-)\|_{2}^{\beta}\right)^{1/\beta}\right).
\end{multline*}
By applying the Lipschitz condition in assumption $\mathrm{(\nameref{ass:Theta})}$,
\begin{multline}
\left|\ell(\phi,\mathbf{z}_i,k,\tau) - \ell(\phi^\prime,\mathbf{z}_i,k,\tau)  \right| \\
\leq \frac{4}{\tau} C_{\Phi} \|\theta-\theta^\prime\|_2\left(\|z_i^a\|_2+\|z_i^+\|_2+e^{2/\beta\tau}\left(\frac{1}{k}\sum_{j=1}^k  \|z_{ij}^- \|_{2}^{\beta}\right)^{1/\beta}\right).\label{eq:ell_diff_proof}
\end{multline}
Applying the triangle inequality for the $\|\cdot\|_2$ norm in $\R^n$,
\begin{multline}\label{eq:data_bound}
\left(\frac{1}{n}\sum_{i=1}^n\left[\|z_i^a\|_2+\|z_i^+\|_2+e^{2/\beta\tau}\left(\frac{1}{k}\sum_{j=1}^k  \|z_{ij}^- \|_{2}^{\beta}\right)^{1/\beta}\right]^2\right)^{1/2}\\
\leq \left(\frac{1}{n}\sum_{i=1}^n\|z_i^a\|_2^2\right)^{1/2} +\left(\frac{1}{n}\sum_{i=1}^n\|z_i^+\|_2^2\right)^{1/2} +e^{2/\beta\tau}\left[\frac{1}{n}
\sum_{i=1}^n\left(\frac{1}{k}\sum_{j=1}^k\|z_{ij}^-\|_2^{\beta}\right)^{2/\beta}\right]^{1/2}.
\end{multline}

Now set $\beta = 1 \vee 2/\tau$. When $\tau<1$ we have $\tau =2/\beta<1$ and $\beta\tau = \tau\vee 2=2$. In this case, by Jensen's inequality,
$$
e^{2/\beta\tau}\left[\frac{1}{n}
\sum_{i=1}^n\left(\frac{1}{k}\sum_{j=1}^k\|z_{ij}^-\|_2^{\beta}\right)^{2/\beta}\right]^{1/2} \leq  
e \left(\frac{1}{nk}
\sum_{i=1}^n \sum_{j=1}^k\|z_{ij}^-\|_2^{\beta}\right)^{1/\beta} = e \left(\frac{1}{nk}
\sum_{i=1}^n\sum_{j=1}^k\|z_{ij}^-\|_2^{2/\tau}\right)^{\tau/2}.
$$

On the other hand, when $\tau \geq 1$, we have $\beta\in[1,2]$, $2/\beta \geq 1$ and $\beta\tau = \tau \vee 2\geq 2$. In this case, by Jensen's inequality,
$$
e^{2/\beta\tau}\left[\frac{1}{n}
\sum_{i=1}^n\left(\frac{1}{k}\sum_{j=1}^k\|z_{ij}^-\|_2^{\beta}\right)^{2/\beta}\right]^{1/2} \leq  
e \left[\frac{1}{nk}
\sum_{i=1}^n \sum_{j=1}^k\|z_{ij}^-\|_2^{2}\right]^{1/2}. $$
Combining the two cases: $\tau<1$ and $\tau\geq 1$,
\begin{equation}\label{eq:tau_both_cases}
e^{2/\beta\tau}\left[\frac{1}{n}
\sum_{i=1}^n\left(\frac{1}{k}\sum_{j=1}^k\|z_{ij}^-\|_2^{\beta}\right)^{2/\beta}\right]^{1/2} \leq e \left[\frac{1}{nk}
\sum_{i=1}^n \sum_{j=1}^k\|z_{ij}^-\|_2^{2/(1\wedge\tau)}\right]^{(1\wedge\tau)/2}.
\end{equation}
On both sides of \eqref{eq:ell_diff_proof}, take the square, then the arithmetic average over the index $i=1,\ldots,n$, then take the square root. Combined with \eqref{eq:data_bound} and \eqref{eq:tau_both_cases} this gives:
$$
\rho^{\mathrm{info}}_n(\phi_\theta,\phi_{\theta^\prime}) \\ \leq \frac{4e}{\tau} B_\tau(\mathbf{z}_1,\ldots,\mathbf{z}_n) C_\Phi  \|\theta-\theta^\prime\|_2. 
$$
\end{proof}

\begin{proof}[Proof of proposition \ref{prop:prop_dudley_apply}.] Let $\Phi=\{\phi_\theta;\theta\in\Theta\}$ as per assumption $\mathrm{(\nameref{ass:Theta})}$.  As shorthand notation, let us absorb various quantities in the statement of lemma \ref{lem:metric} into a constant $L$ such that for all $\theta,\theta^\prime \in\Theta$:
\begin{equation}\label{eq:L_lipschitz_shortand}
\rho^{\mathrm{info}}(\phi_\theta,\phi_{\theta^\prime} ) \leq L \|\theta-\theta^\prime\|_2.
\end{equation}

If for some $N\geq 1$ and $\epsilon>0$, $\{\theta^1,\ldots,\theta^N\}$ is an $\epsilon/L$-cover of $\Theta$ with respect to the $\|\cdot-\cdot\|_2$ distance, then it follows from \eqref{eq:L_lipschitz_shortand} that $\{\phi_{\theta^1},\ldots,\phi_{\theta^N}\}$ is an $\epsilon$-cover of $\Phi$ with respect to $\rho^{\mathrm{info}}$. In turn, the associated covering numbers of $\Phi$ and $\Theta$ obey:
$$
\mathcal{N}(\epsilon,\Phi,\rho^{\mathrm{info}}) \leq \mathcal{N}(\epsilon/L,\Theta,\|\cdot-\cdot\|_2).
$$
Combining this inequality with Dudley's entropy integral, e.g., \citep[eq. 5.48]{wainwright2019high},
\begin{align}
&\mathbb{E}\left[ \sup_{\phi\in\Phi}\frac{1}{n}\sum_{i=1}^n\sigma_i \ell(\phi,\mathbf{z}_i,k,\tau)  \right]\nonumber \\
&\leq  \frac{24}{\sqrt{n}} \int_{0}^{\frac{4}{\tau}\wedge LR} \sqrt{\log\mathcal{N}(\epsilon,\Phi,\rho^{\mathrm{info}})}\, \d \epsilon\nonumber \\
&\leq \frac{24}{\sqrt{n}} \int_{0}^{\frac{4}{\tau}\wedge LR} \sqrt{\log\mathcal{N}(\epsilon/L,\Theta,\|\cdot-\cdot\|_2)}\, \d \epsilon, \label{eq:integral_bound}
\end{align}
where $\sigma_1,\ldots,\sigma_n$ are i.i.d. Rademacher variables; the integral upper-limit term $4/\tau$ in the first inequality holds because by lemma \ref{lem:ell_determ_bound}, $\rho^{\mathrm{info}}(\phi,\phi^\prime)\leq 4/\tau$ for all $\phi,\phi^\prime\in\Phi\subset\BS$; the integral upper-limit $LR$ appears since by combining \eqref{eq:L_lipschitz_shortand} with the definition of $\Theta$ to we have $\rho^{\mathrm{info}}(\phi_{\theta},\phi_{\theta^\prime})\leq LR$.

We have the standard Euclidean volumetric estimate: $\mathcal{N}(\epsilon,\Theta,\|\cdot-\cdot\|_2) \leq (1+2R/\epsilon)^{d_{\Theta}} $, see, e.g., \citep[eq. 5.9]{wainwright2019high}. Writing $U\coloneqq 4/\tau \wedge LR  $, Cauchy-Schwartz gives:
\begin{align*}
\int_0^U \sqrt{\log\mathcal{N}(\epsilon/L,\Theta,\|\cdot-\cdot\|_2)}\, \d \epsilon  &\leq \left(\int_{0}^U 1^2\, \d \epsilon \right)^{1/2} \left(\int_0^U \log\mathcal{N}(\epsilon/L,\Theta,\|\cdot-\cdot\|_2) \right)^{1/2} \, \d \epsilon  \\
&\leq  \sqrt{U d_{\Theta}} \left(\int_0^U\log(1+2RL/\epsilon)\, \d \epsilon\right)^{1/2}\\
& = \sqrt{U d_{\Theta}} \left[U \log\left(1+\frac{2RL}{U}\right) + 2RL \log\left(1+\frac{U}{2RL}\right) \right]^{1/2} \\
&\leq U\sqrt{ d_{\Theta}} \left[ \log\left(1+\frac{2RL}{U}\right) + 1 \right]^{1/2}, 
\end{align*}
where the equality uses the fact that the anti-derivative in question is $\epsilon \log(1+2RL/\epsilon) +2RL\log(1+\epsilon/2RL)$, and the final inequality uses $\log(1+x)\leq x$ for $x\geq 0$.

Returning to \eqref{eq:integral_bound} in the case $4/\tau \leq RL$, we have $U=4/\tau$ and:
\begin{align*}
\mathbb{E}\left[ \sup_{\phi\in\Phi}\frac{1}{n}\sum_{i=1}^n\sigma_i \ell(\phi,\mathbf{z}_i,k,\tau)  \right] &\leq \frac{24\sqrt{d_{\Theta}}}{\sqrt{n}}\frac{4}{\tau} \left[\log\left(1+\frac{2RL}{4/\tau}\right)+1\right]^{1/2}\\
&=  \frac{24\sqrt{d_{\Theta}}}{\sqrt{n}}\frac{4}{\tau} \left[\log\left(1+2a \right)+1\right]^{1/2},
\end{align*}
where  $L = 4eB_{\tau}(\mathbf{z}_1,\ldots,\mathbf{z}_n)C_{\Phi}/\tau$ (from lemma \ref{lem:metric}) has been used and $a\coloneqq eR B_{\tau}(\mathbf{z}_1,\ldots,\mathbf{z}_n)C_{\Phi} $. 

On the other hand, if $4/\tau > RL$, we have $U=RL$ and 
\begin{align*}
\mathbb{E}\left[ \sup_{\phi\in\Phi}\frac{1}{n}\sum_{i=1}^n\sigma_i \ell(\phi,\mathbf{z}_i,k,\tau)  \right] & \leq \frac{24\sqrt{d_{\Theta}}}{\sqrt{n}}RL \left[\log(3)+1\right]^{1/2}  \\
&\leq \frac{24\sqrt{d_{\Theta}}}{\sqrt{n}}\frac{4}{\tau}a[\log(3)+1]^{1/2}.
\end{align*}

Since $4/\tau >RL \Leftrightarrow 1 > eR B_{\tau}(\mathbf{z}_1,\ldots,\mathbf{z}_n)C_{\Phi}\Leftrightarrow 1>a$, we obtain
$$
\mathbb{E}\left[ \sup_{\phi\in\Phi}\frac{1}{n}\sum_{i=1}^n\sigma_i \ell(\phi,\mathbf{z}_i,k,\tau)  \right] \leq \frac{96}{\tau}\sqrt{\frac{d_{\Theta}}{n}}\min\left\{a\sqrt{\log(3)+1},\sqrt{\log(1+2a)+1}\right\}.
$$

\end{proof}

\bibliographystyle{plainnat}
\bibliography{refs_CL}

\begin{thebibliography}{25}
\providecommand{\natexlab}[1]{#1}
\providecommand{\url}[1]{\texttt{#1}}
\expandafter\ifx\csname urlstyle\endcsname\relax
  \providecommand{\doi}[1]{doi: #1}\else
  \providecommand{\doi}{doi: \begingroup \urlstyle{rm}\Url}\fi

\bibitem[Bartlett et~al.(2017)Bartlett, Foster, and
  Telgarsky]{bartlett2017spectrally}
Peter~L Bartlett, Dylan~J Foster, and Matus~J Telgarsky.
\newblock Spectrally-normalized margin bounds for neural networks.
\newblock \emph{{Advances in Neural Information Processing Systems}}, 30, 2017.

\bibitem[Chen et~al.(2020)Chen, Kornblith, Norouzi, and Hinton]{chen2020simclr}
Ting Chen, Simon Kornblith, Mohammad Norouzi, and Geoffrey Hinton.
\newblock A simple framework for contrastive learning of visual
  representations.
\newblock In \emph{Proceedings of the 37th International Conference on Machine
  Learning}, volume 119 of \emph{Proceedings of Machine Learning Research},
  pages 1597--1607. {PMLR}, 2020.

\bibitem[Ghanooni et~al.(2024)Ghanooni, Mustafa, Lei, Lin, and
  Kloft]{ghanooni2024generalization}
Naghmeh Ghanooni, Waleed Mustafa, Yunwen Lei, Anthony~Widjaja Lin, and Marius
  Kloft.
\newblock Generalization bounds with logarithmic negative-sample dependence for
  adversarial contrastive learning.
\newblock \emph{Transactions on Machine Learning Research}, 2024.

\bibitem[Golowich et~al.(2018)Golowich, Rakhlin, and Shamir]{golowich2018size}
Noah Golowich, Alexander Rakhlin, and Ohad Shamir.
\newblock Size-independent sample complexity of neural networks.
\newblock In \emph{Proceedings of the 31st Conference On Learning Theory},
  pages 297--299. PMLR, 2018.

\bibitem[Gonon et~al.(2025)Gonon, Brisebarre, Riccietti, and
  Gribonval]{gonon2025rescaling}
Antoine Gonon, Nicolas Brisebarre, Elisa Riccietti, and R{\'e}mi Gribonval.
\newblock A rescaling-invariant lipschitz bound based on path-metrics for
  modern relu network parameterizations.
\newblock In \emph{International Conference on Machine Learning}, pages
  20047--20074. PMLR, 2025.

\bibitem[He et~al.(2020)He, Fan, Wu, Xie, and Girshick]{he2020moco}
Kaiming He, Haoqi Fan, Yuxin Wu, Saining Xie, and Ross Girshick.
\newblock Momentum contrast for unsupervised visual representation learning.
\newblock In \emph{Proceedings of the {IEEE/CVF} Conference on Computer Vision
  and Pattern Recognition ({CVPR})}, pages 9729--9738. {IEEE}, 2020.
\newblock \doi{10.1109/cvpr42600.2020.00975}.

\bibitem[Henaff(2020)]{henaff2020data}
Olivier Henaff.
\newblock Data-efficient image recognition with contrastive predictive coding.
\newblock In \emph{International Conference on Machine Learning}, pages
  4182--4192. PMLR, 2020.

\bibitem[Hieu and Ledent(2025)]{hieu2025generalization}
Nong~Minh Hieu and Antoine Ledent.
\newblock Generalization analysis for supervised contrastive representation
  learning under non-iid settings.
\newblock In \emph{International Conference on Machine Learning}, pages
  23179--23218. PMLR, 2025.

\bibitem[Khosla et~al.(2020)Khosla, Teterwak, Wang, Sarna, Tian, Isola,
  Maschinot, Liu, and Krishnan]{khosla2020supervised}
Prannay Khosla, Piotr Teterwak, Chen Wang, Aaron Sarna, Yonglong Tian, Phillip
  Isola, Aaron Maschinot, Ce~Liu, and Dilip Krishnan.
\newblock Supervised contrastive learning.
\newblock \emph{{Advances in Neural Information Processing Systems}},
  33:\penalty0 18661--18673, 2020.

\bibitem[Lei et~al.(2019)Lei, Dogan, Zhou, and Kloft]{lei2019data}
Yunwen Lei, {\"U}r{\"u}n Dogan, Ding-Xuan Zhou, and Marius Kloft.
\newblock Data-dependent generalization bounds for multi-class classification.
\newblock \emph{IEEE Transactions on Information Theory}, 65\penalty0
  (5):\penalty0 2995--3021, 2019.

\bibitem[Lei et~al.(2023)Lei, Yang, Ying, and Zhou]{lei2023generalization}
Yunwen Lei, Tianbao Yang, Yiming Ying, and Ding-Xuan Zhou.
\newblock Generalization analysis for contrastive representation learning.
\newblock In \emph{International Conference on Machine Learning}, pages
  19200--19227. PMLR, 2023.

\bibitem[Maurer(2016)]{maurer2016vector}
Andreas Maurer.
\newblock A vector-contraction inequality for rademacher complexities.
\newblock In \emph{International Conference on Algorithmic Learning Theory},
  pages 3--17. Springer, 2016.

\bibitem[Mohri et~al.(2018)Mohri, Rostamizadeh, and
  Talwalkar]{mohri2018foundations}
Mehryar Mohri, Afshin Rostamizadeh, and Ameet Talwalkar.
\newblock \emph{Foundations of Machine Learning}.
\newblock MIT press, 2018.

\bibitem[Neyshabur et~al.(2015)Neyshabur, Tomioka, and
  Srebro]{neyshabur2015norm}
Behnam Neyshabur, Ryota Tomioka, and Nathan Srebro.
\newblock Norm-based capacity control in neural networks.
\newblock In \emph{Conference on Learning Theory}, pages 1376--1401. PMLR,
  2015.

\bibitem[Radford et~al.(2021)Radford, Kim, Hallacy, Ramesh, Goh, Agarwal,
  Sastry, Askell, Mishkin, Clark, Krueger, and Sutskever]{radford2021clip}
Alec Radford, Jong~Wook Kim, Chris Hallacy, Aditya Ramesh, Gabriel Goh,
  Sandhini Agarwal, Girish Sastry, Amanda Askell, Pamela Mishkin, Jack Clark,
  Gretchen Krueger, and Ilya Sutskever.
\newblock Learning transferable visual models from natural language
  supervision.
\newblock In \emph{Proceedings of the 38th International Conference on Machine
  Learning}, volume 139 of \emph{Proceedings of Machine Learning Research},
  pages 8748--8763. {PMLR}, 2021.
\newblock \doi{10.48550/arxiv.2103.00020}.

\bibitem[Saunshi et~al.(2019)Saunshi, Plevrakis, Arora, Khodak, and
  Khandeparkar]{saunshi2019theoretical}
Nikunj Saunshi, Orestis Plevrakis, Sanjeev Arora, Mikhail Khodak, and
  Hrishikesh Khandeparkar.
\newblock A theoretical analysis of contrastive unsupervised representation
  learning.
\newblock In \emph{International Conference on Machine Learning}, pages
  5628--5637. PMLR, 2019.

\bibitem[Tian et~al.(2020)Tian, Krishnan, and Isola]{tian2020contrastive}
Yonglong Tian, Dilip Krishnan, and Phillip Isola.
\newblock Contrastive multiview coding.
\newblock In \emph{European conference on computer vision}, pages 776--794.
  Springer, 2020.

\bibitem[van~den Oord et~al.(2018)van~den Oord, Li, and
  Vinyals]{vandenOord2018cpc}
A{\"a}ron van~den Oord, Yazhe Li, and Oriol Vinyals.
\newblock Representation learning with contrastive predictive coding.
\newblock \emph{arXiv preprint arXiv:1807.03748}, 2018.

\bibitem[Wainwright(2019)]{wainwright2019high}
Martin~J Wainwright.
\newblock \emph{High-dimensional statistics: A non-asymptotic viewpoint},
  volume~48.
\newblock Cambridge university press, 2019.

\bibitem[Wang and Isola(2020{\natexlab{a}})]{wang2020understanding}
Tongzhou Wang and Phillip Isola.
\newblock Understanding contrastive representation learning through alignment
  and uniformity on the hypersphere.
\newblock In \emph{Proceedings of the 37th International Conference on Machine
  Learning}. PMLR, 2020{\natexlab{a}}.

\bibitem[Wang and Isola(2020{\natexlab{b}})]{wang2020understandingv10}
Tongzhou Wang and Phillip Isola.
\newblock Understanding contrastive representation learning through alignment
  and uniformity on the hypersphere, 2020{\natexlab{b}}.
\newblock URL \url{https://arxiv.org/abs/2005.10242v10}.

\bibitem[Wu et~al.(2018)Wu, Xiong, Yu, and Lin]{wu2018unsupervised}
Zhirong Wu, Yuanjun Xiong, Stella~X Yu, and Dahua Lin.
\newblock Unsupervised feature learning via non-parametric instance
  discrimination.
\newblock In \emph{Proceedings of the IEEE conference on computer vision and
  pattern recognition}, pages 3733--3742, 2018.

\bibitem[Yang et~al.(2021)Yang, Hu, Babuschkin, Sidor, Liu, Farhi, Ryder,
  Pachocki, Chen, and Gao]{yang2021tuning}
Ge~Yang, Edward Hu, Igor Babuschkin, Szymon Sidor, Xiaodong Liu, David Farhi,
  Nick Ryder, Jakub Pachocki, Weizhu Chen, and Jianfeng Gao.
\newblock Tuning large neural networks via zero-shot hyperparameter transfer.
\newblock \emph{{Advances in Neural Information Processing Systems}},
  34:\penalty0 17084--17097, 2021.

\bibitem[Yeh et~al.(2022)Yeh, Hong, Hsu, Liu, Chen, and
  LeCun]{yeh2022decoupled}
Chun-Hsiao Yeh, Cheng-Yao Hong, Yen-Chi Hsu, Tyng-Luh Liu, Yubei Chen, and Yann
  LeCun.
\newblock Decoupled contrastive learning.
\newblock In \emph{European conference on computer vision}, pages 668--684.
  Springer, 2022.

\bibitem[Zimmermann et~al.(2021)Zimmermann, Sharma, Schneider, Bethge, and
  Brendel]{zimmermann2021contrastive}
Roland~S. Zimmermann, Yash Sharma, Steffen Schneider, Matthias Bethge, and
  Wieland Brendel.
\newblock Contrastive learning inverts the data generating process.
\newblock In \emph{International Conference on Machine Learning}, pages
  12979--12990. PMLR, 2021.

\end{thebibliography}

\end{document}